\documentclass[lettersize,journal]{IEEEtran}
\usepackage{amsmath,amsfonts}
\usepackage{algorithmic}
\usepackage{algorithm}
\usepackage{array}
\usepackage{textcomp}
\usepackage{stfloats}
\usepackage{url}
\usepackage{verbatim}
\usepackage{graphicx}
\usepackage{cite}
\usepackage{subfigure}

\usepackage{hyperref}       
\usepackage{url}            
\usepackage{booktabs}       
\usepackage{nicefrac}       
\usepackage{microtype}      

\usepackage{times}
\usepackage{epsfig}
\usepackage{graphicx}
\usepackage{epstopdf}
\usepackage{amsmath}
\usepackage{amssymb}

\usepackage{amsthm, color, listings, comment}
\usepackage{bm, hyperref}
\hypersetup{
    colorlinks=true,
    linkcolor=blue,
    filecolor=magenta,      
    urlcolor=cyan,
}

\newcommand{\redt}[1]{\textcolor{red}{#1}}

\newcommand{\norm}[1]{\|#1\|}
\newtheorem{lemma}{Lemma}
\newtheorem{theorem}{Theorem}

\newtheorem{corollary}{Corollary}
\newtheorem{remark}{Remark}
\newtheorem*{theorem*}{Theorem}   
\newtheorem*{proposition*}{Proposition}   

\begin{document}

\title{A Unified Analysis of AdaGrad with Weighted Aggregation and Momentum Acceleration}

\author{Li Shen, Congliang Chen, Fangyu Zou, Zequn Jie, Ju Sun, Wei Liu,~\IEEEmembership{Fellow,~IEEE}
\thanks{This work is supported by Science and Technology Innovation 2030 –“Brain Science and Brain-like Research” key Project (No. 2021ZD0201405).}
\thanks{Li. Shen is with the JD Explore Academy, Beijing, China (email: mathshenli@gmail.com).}
\thanks{Congliang Chen is with The Chinese University of Hong Kong, Shenzhen, China (email: chcoli2007@163.com).}
\thanks{Fangyu Zou is with Facebook, USA (email: fangyuzou@gmail.com).}
\thanks{Zequn Jie is with Meituan, Beijing, China (email: zequn.nus@gmail.com).}
\thanks{Ju Sun is with University of Minnesota, Twin Cities, USA  (email: sunjunus@gmail.com).}
\thanks{Wei Liu is with Tencent, Shenzhen, China (email:wl2223@columbia.edu).}
\thanks{Manuscript received July 29, 2022; revised January 08, 2023; Accepted May 15, 2023}
}

\markboth{Journal of \LaTeX\ Class Files,~Vol.~14, No.~8, August~2021}%
{Shell \MakeLowercase{\textit{et al.}}: A Sample Article Using IEEEtran.cls for IEEE Journals}


\maketitle

\begin{abstract}
Integrating adaptive learning rate and momentum techniques into SGD leads to a large class of efficiently accelerated adaptive stochastic algorithms, such as AdaGrad, RMSProp, Adam, AccAdaGrad, \textit{etc}. In spite of their effectiveness in practice, there is still a large gap in their theories of convergences, especially in the difficult non-convex stochastic setting. To fill this gap, we propose \emph{weighted AdaGrad with unified momentum}, dubbed AdaUSM, which has the main characteristics that (1) it incorporates a unified momentum scheme which covers both the heavy ball momentum and the Nesterov accelerated gradient momentum; (2) it adopts a novel weighted adaptive learning rate that can unify the learning rates of AdaGrad, AccAdaGrad, Adam, and RMSProp. Moreover, when we take polynomially growing weights in AdaUSM, we obtain its $\mathcal{O}(\log(T)/\sqrt{T})$ convergence rate in the non-convex stochastic setting. We also show that the adaptive learning rates of Adam and RMSProp correspond to taking exponentially growing weights in AdaUSM, thereby providing a new perspective for understanding Adam and RMSProp. Lastly, comparative experiments of AdaUSM against SGD with momentum, AdaGrad, AdaEMA, Adam, and AMSGrad on various deep learning models and datasets are also carried out.
\end{abstract}


\section{Introduction}
\IEEEPARstart{I}{n} this work we consider the following general non-convex stochastic optimization problem: 
\begin{align}\label{population-risk}
\min_{x\in \mathbb{R}^{d}}\ f(x) := \mathbb{E}_{z}\,\big[F(x,z)\big],
\end{align}
where $\mathbb{E}_{z}[\cdot]$ denotes the expectation with respect to random variable $z$.
We assume that $f$ is bounded from below, \emph{i.e.}, $f^* = \inf_{x\in\mathbb{R}^d} f(x) > \infty$, and its gradient $\nabla{f}(\cdot)$ is $L$-Lipschitz continuous.  

Problem \eqref{population-risk} arises from many statistical learning (\textit{e.g.}, logistic regression, AUC maximization) and deep learning models \cite{goodfellow2016deep,lecun2015deep}. 
In general, one only has access to noisy estimates of $\nabla f$, as the expectation in problem \eqref{population-risk} can often only be approximated as a finite sum. 
Hence, one of the most popular algorithms to solve problem \eqref{population-risk} is Stochastic Gradient Decent (SGD) \cite{robbins1985stochastic,bottou2018optimization}:
\begin{equation}\label{SGD}
x_{t+1} := x_t - \eta_t g_t, 
\end{equation}
where $\eta_t$ is the learning rate and $g_t$ is the noisy gradient estimate of $\nabla{f}(x)$ in the $t$-th iteration.
Its convergence rates for both convex and non-convex settings have been established in \cite{bottou1998online,ghadimi2013stochastic}. 

However, vanilla SGD suffers from slow convergence, and its performance is sensitive to the learning rate---which is tricky to tune. 
Many techniques have been introduced to improve the convergence speed and robustness of SGD,  such as variance reduction~\cite{defazio2014saga,johnson2013accelerating,nguyen2017sarah}, adaptive learning rate~\cite{duchi2011adaptive,kingma2014adam}, and momentum acceleration~\cite{polyak1964some,nesterov1983method,levy2018online}. Among them, adaptive learning rate and momentum acceleration techniques are most economic, since they merely require slightly more computations per iteration. 
SGD with adaptive learning rate was first proposed as AdaGrad~\cite{mcmahan2010adaptive,duchi2011adaptive} and the learning rate is adjusted by cumulative gradient magnitudes: 
\begin{equation}\label{average-learning-rate}
\eta_t = \frac{\eta}{\sqrt{\textstyle\sum_{i=1}^t g_i^2} + \epsilon},
\end{equation}
where $\epsilon, \eta > 0$ are fixed parameters. 
On the other hand, Heavy Ball (HB)~\cite{polyak1964some,ghadimi2015global} and Nesterov Accelerated Gradient (NAG)~\cite{nesterov1983method,nesterov2013introductory} are two most popular momentum acceleration techniques, which have been extensively studied for stochastic optimization problems \cite{ghadimi2016accelerated,yan2018unified,levy2018online}: 
\begin{gather}\label{two-momentum}
\textbf{(SHB)}:\ \left\{\ 
\begin{aligned}
& m_t = \mu m_{t-1} - \eta_t g_{t}\\
& x_{t+1} = x_t + m_t
\end{aligned}, 
\right.\\ 
 \textbf{(SNAG)}:\ \left\{
\begin{aligned}
& y_{t+1} = x_{t} - \eta_t g_{t} \\
& x_{t+1} = y_{t+1} + \mu (y_{t+1}-y_t)
\end{aligned}
\right. , 
\end{gather}
where $x_1=y_1\in\mathbb{R}^d$, $m_{0}=\bm{0}\in\mathbb{R}^{d}$, and $\mu \in [0, 1)$ is the momentum factor. 

Both the adaptive learning rate and momentum techniques have been individually investigated and have displayed to be effective in practice, so they are independently and widely applied in tasks such as training deep networks~\cite{krizhevsky2012imagenet,sutskever2013importance,kingma2014adam,reddi2018convergence}.
It is natural to consider: \emph{Can we effectively incorporate both techniques at the same time so as to inherit their dual advantages and moreover develop convergence theory for this scenario, especially in the more difficult non-convex stochastic setting?}
It is expected that there will be new difficulties in the analysis which comes from the combination of both techniques,
comparing to merely using adaptive learning rate or merely using momentum technique.
To the best of our knowledge, Levy \textit{et al.} \cite{levy2018online} firstly attempted to combine the adaptive learning rate with NAG momentum, which yields the AccAdaGrad algorithm. However, its convergence is limited to the stochastic convex setting. 
Yan \textit{et al.} \cite{yan2018unified} unified SHB and SNAG to a three-step iterate without considering the adaptive learning rate in Eq.~\eqref{average-learning-rate}. 

In this work, we revisit the momentum acceleration technique~\cite{polyak1964some,nesterov1983method} and adaptive learning rate~\cite{duchi2011adaptive,kingma2014adam}, and propose weighted AdaGrad with unified stochastic momentum, dubbed AdaUSM, to solve the general non-convex stochastic optimization problem \eqref{population-risk}. Specifically, the proposed AdaUSM has two main features: it develops a novel {U}nified {S}tochastic {M}omentum (USM) scheme to cover SHB and SNAG, entirely different from the three-step scheme in~\cite{yan2018unified}, and it generalizes the adaptive learning rate in Eq.~\eqref{average-learning-rate} to a more general weighted adaptive learning rate (see Section \ref{sec-AdaUSM}) that can unify the adaptive learning rates of AdaGrad, AccAdaGrad, and Adam into a succinct framework.  
In contrast to those in AdaGrad \cite{duchi2011adaptive}, the weighted adaptive learning rate in AdaUSM is estimated via a novel weighted gradient accumulation technique, which puts more weights on the most recent stochastic gradient estimates. 
Moreover, to make AdaUSM more practical for large-scale problems, a coordinate-wise weighted adaptive learning rate with a low computational cost is used. 

We also characterize the $\mathcal{O}\big(\log(T)/\sqrt{T}\big)$ convergence rate of AdaUSM in the non-convex stochastic setting, when we take polynomially growing weights in AdaUSM. 
When momentum is NAG and weights are set as the same as those in \cite{levy2018online}, AdaUSM reduces to AccAdaGrad \cite{levy2018online}. In consequence, the convergence rate of AccAdaGrad in the non-convex setting is derived directly as a byproduct. 
Thus, our work generalizes AccAdaGrad \cite{levy2018online} in three aspects: (i) more general weights in estimating the adaptive learning rate; (ii) new unified momentum including both NAG and HB; (iii) the convergence rate in the more difficult non-convex stochastic setting.  
Our contributions are three-fold: 
\begin{itemize} 
\item We develop a new weighted gradient accumulation technique to estimate the adaptive learning rate,
and propose a novel unified stochastic momentum scheme to cover SHB and SNAG.
We then integrate the weighted coordinate-wise AdaGrad with a unified momentum mechanism, yielding a novel adaptive stochastic momentum algorithm, dubbed AdaUSM. 

\item We establish the $\mathcal{O}({\log{(T)}}{/}{\sqrt{T}})$ non-asymptotic convergence rate of AdaUSM under the general non-convex stochastic setting. Our assumptions are natural and mild. 

\item We show that the adaptive learning rates of Adam and RMSProp correspond to taking exponentially growing weights in AdaUSM, which thereby provides a new perspective for understanding Adam and RMSProp.

\end{itemize}

{\bf Related Works.}\quad
There exist several works to study the convergence rates of adaptive SGD in the non-convex stochastic setting. All of the related works share the same assumption on the objective function that it has Lipschitz gradients and unbiased estimation for gradients. Still, they have different assumptions on the second-order statistic information of stochastic gradients. Specifically, Li and Orabona \cite{li2018convergence} first proved the global convergence of perturbed AdaGrad by assuming the variance of stochastic gradients is bounded. Manzil Zaheer \textit{et al.}\cite{zaheer2018adaptive} proved that Adam could converge near to the stationary point where the ratio is related to the noise level. Ward \textit{et al.}~\cite{ward2018adagrad} established $\mathcal{O}(\log{(T)}/\sqrt{T})$ convergence rates with the measure $(\mathbb{E} (\|\nabla f(x_t)\|^{4/3}))^{3/2}$for the original AdaGrad \cite{duchi2011adaptive} and WNgrad \cite{wu2018wngrad} by adding a bounded ture gradient assumption.
The same convergence rates of Adam/RMSProp are established in \cite{zou2018sufficient,chen2022towards} with the same assumptions in \cite{duchi2011adaptive}. AMSGrad \cite{chen2018convergence} were also established in the non-convex stochastic setting with directly assuming bounded stochastic gradients in the same rate, but the measure is improved to $\mathbb{E} \|\nabla f(x_t)\|^2$. Besides, \cite{guo2021novel} established $\mathcal{O}(\log{(T)}/\sqrt{T})$ with the same measure as it in \cite{duchi2011adaptive} for Adam with the assumption that the elements in adaptive stepsize are upper and lower bounded. We list the detailed comparison of these results in Table \ref{tab1}.

\begin{table}[H]
\label{tab1}
\caption{Comparison of different analyses.
Here we consider the assumptions on the second-order information of stochastic gradients, where "BG" denotes true gradients are bounded, "BSG" denotes stochastic gradients are uniformly bounded, "BV" denotes  the variance of stochastic gradients is bounded, and "B$\eta$" denotes $\eta_t$'s are upper and lower bounded. "*" denotes convergence for $(\mathbb{E} \|\nabla f(x_t)\|^{4/3})^{3/2}$.}
\begin{center}
\begin{tabular}{ccc}
\hline
Algorithm & Assumption & Convergence Rate\\
\hline
AdaGrad\cite{li2018convergence} & BV & Converge\\
Adam\cite{zaheer2018adaptive} & BV & Converge to neighborhood\\
AdaGrad\cite{ward2018adagrad}& BG+BV & $\mathcal{O}(logT/\sqrt{T})^*$\\
WNgrad\cite{ward2018adagrad}& BG+BV & $\mathcal{O}(logT/\sqrt{T})^*$\\
Adam\cite{zou2018sufficient}&BG+BV &$\mathcal{O}(logT/\sqrt{T})^*$\\
RMSProp\cite{zou2018sufficient}&BG+BV &$\mathcal{O}(logT/\sqrt{T})^*$\\
AMSGrad\cite{chen2018convergence} & BSG &$\mathcal{O}(logT/\sqrt{T})$\\
Adam\cite{guo2021novel} & B$\eta$ &$\mathcal{O}(logT/\sqrt{T})$\\
AdaUSM & BG+BV & $\mathcal{O}(logT/\sqrt{T})^*$
\end{tabular}
\end{center}
\end{table}

\section{Preliminaries}
\paragraph{Notations.} 
$T$ denotes the maximum number of iterations.
The noisy gradient of $f$ at the $t$-th iteration is denoted by $g_t$ for all $t=1,2,\cdots, T$. 
We use $\mathbb{E}[\cdot]$ to denote expectation as usual, and $\mathbb{E}_t[\cdot]$ as the conditional expectation with respect to $g_t$ conditioned on the random variables $\{g_1, g_2, \cdots, g_{t-1}\}$.

In this paper, we allow differential learning rates across coordinates, so the learning rate $\eta_t$ is a vector in $\mathbb{R}^d$. Given a vector $v \in \mathbb{R}^d$ we denote its $k$-th coordinate by $v_{k}$. 
The $k$-th coordinate of the gradient $\nabla f(x)$ is denoted by $\nabla _k f(x)$. 
Given two vectors $v_t, w_t \in \mathbb{R}^d$, the inner product between them is denoted by $\langle v_t, w_t \rangle :=\sum_{k=1}^d v_{t,k} w_{t,k}$. We also heavily use the coordinate-wise product between $v$ and $w$, denoted as $vw \in \mathbb R^d$, with $(vw)_k = v_k w_k$. 
Division by a vector is defined similarly. 
Given a vector $v\!\in\! \mathbb{R}^{d}$, we define the weighted norm:
$\norm{\nabla f(x)}_v^2:=\textstyle\sum_{k=1}^d v_k |\nabla_k f(x)|^2.$
Norm $\norm{\cdot}$ without any subscript is the Euclidean norm, and $\norm{\cdot}_1$ is defined as $\norm{v}_1 = \sum_{k=1}^d |v_k|$. 
Let $\bm{0} = (0,\ldots,0)^\top \in \mathbb{R}^{d}$ and $\bm{\epsilon} = (\epsilon,\ldots,\epsilon)^\top \in\mathbb{R}^d$.

\paragraph{Assumptions.} 
We assume that $g_t$, $t= 1, 2, \ldots, T$ are independent of each other. Moreover, 
\begin{itemize} 
\item[(\textbf{A1})] $\mathbb{E}_t\left[g_t\right]=\nabla f(x_t)$, \textit{i.e.}, $g_t$ is an unbiased estimator;
\item[(\textbf{A2})] $\mathbb{E}_t\norm{g_t}^2 \leq\sigma^2$, \textit{i.e.}, the second-order moment of $g_{t}$ is bounded.
\end{itemize}
Notice that the condition (\textbf{A2}) is slightly weaker than that in \cite{chen2018convergence} which assumes that stochastic gradient estimate $g_{t}$ is uniformly bounded, \textit{i.e.}, $\|g_{t}\|\le \sigma$.

\section{Weighted AdaGrad with Unified Momentum}\label{sec-AdaUSM}

 We describe the two main ingredients of AdaUSM: the unified stochastic momentum formulation of SHB and SNAG (see Subsection \ref{sec-USM}), and the weighted adaptive learning rate (see Subsection \ref{sec-weigthed-adagrad}).
 
\subsection{Unified Stochastic Momentum (USM)}\label{sec-USM}
By introducing $m_t = y_{t+1} - y_t$ with $m_0 = 0$, the iterate of SNAG can be equivalently written as
\[
(\textbf{SNAG}):
\left\{
\begin{aligned}
& m_t = \mu m_{t-1} - \eta_t g_t, \\
& x_{t+1} = x_t + m_t + \mu (m_t - m_{t-1}).
\end{aligned}
\right.
\]
Comparing SHB and above SNAG, the difference lies in that SNAG places more emphasis on the current momentum $m_t$. Hence, we can rewrite SHB and SNAG in the following unified form:
\begin{align}\label{USM}
(\textbf{USM}):\ \left\{
\begin{aligned}
& m_t = \mu m_{t-1} - \eta_t g_t, \\
& x_{t+1} = x_t + m_t + \lambda \mu (m_t - m_{t-1}),
\end{aligned}
\right.
\end{align}
where $\lambda \geq 0$ is a constant. When $\lambda = 0$, it is SHB; when $\lambda = 1$, it is SNAG. We call $\lambda$ the interpolation factor. For any $\mu \in [0, 1)$, $\lambda$ can be chosen from $[0,1/(1-\mu)]$. 
\begin{remark}
Yan \textit{et al.} \cite{yan2018unified} unified SHB and SNAG as a three-step iterate scheme as follows:
\begin{gather}
y_{t+1} = x_{t} - \eta_t g_t, \label{UM-ijcai-eq1}\\
y^{s}_{t+1} = x_t -s\eta_{t}g_{t}, \label{UM-ijcai-eq2}\\
x_{t+1} = y_{t+1} + \mu(y^{s}_{t+1} - y^{s}_{t}), \label{UM-ijcai-eq3}
\end{gather}
where $y_{0}^{s} = x_{0}$. Its convergence rate has been established for $\eta_{t} = \mathcal{O}(1/\sqrt{t})$. Notably, USM is slightly simpler than Eqs.~\eqref{UM-ijcai-eq1}-\eqref{UM-ijcai-eq3}, and the learning rate $\eta_{t}$ in USM is adaptively determined. 
\end{remark}

\subsection{Weighted Adaptive Learning Rate}\label{sec-weigthed-adagrad} 
 
We generalize the learning rate in Eq.~\eqref{average-learning-rate} by assigning different weights to the past stochastic gradients accumulated. It is defined as follows:
\begin{align}\label{adaptive-learning-rate} 
\eta_{t,k}  
&=  \frac{\eta}{\sqrt{{\sum_{i=1}^t a_i g_{i,k}^2}/{\bar{a}_t}} + \epsilon}  \nonumber \\
&=  \frac{\eta/\sqrt{t}}{\sqrt{{\sum_{i=1}^t a_i g_{i,k}^2}/({\sum_{i=1}^t a_i}}) + \epsilon/\sqrt{t}},
\end{align}
for $k = 1, 2, \cdots, d$, where $a_1, a_2, \cdots, a_t > 0$ and $\bar{a}_t = \sum_{i=1}^t a_i/t$. Here, $\eta/\sqrt{t}$ can be understood as the base learning rate. 
The classical AdaGrad corresponds to taking $a_t = 1$ for all $t$ in Eq.~\eqref{adaptive-learning-rate}, \textit{i.e.}, uniform weights. However, recent gradients tend to carry more information of local geometries than remote ones. Hence, it is natural to assign the recent gradients more weights. 
A typical choice for such weights is to choose $a_t = t^\alpha$ for $\alpha > 0$, which grows in a polynomial rate. For instance, in AccAdaGrad \cite{levy2018online} weights are chosen to be $a_t=[(1+t)/4]^2$ for $t\ge 3$ and $a_t=1$ for $0 \le t \le 2$, respectively.

\subsection{AdaUSM: Weighted AdaGrad with USM}
In this subsection, we present the AdaUSM algorithm, which effectively integrates the weighted adaptive learning rate in Eq.~\eqref{adaptive-learning-rate} with the USM technique in Eq.~\eqref{USM}, and establish its convergence rate. 

\begin{algorithm}[H]
\caption{\quad AdaUSM: Weighted AdaGrad with USM}
\label{Alg:AdaMoment}
\begin{algorithmic}[1]
   \STATE {\bf Parameters:} Choose $x_1 \!\in\! \mathbb{R}^d$, fixed parameter $\eta \!\ge\! 0$, momentum factor $\mu$, and initial accumulator factor $\epsilon \!>\! 0$. Set $m_0 \!= \!\bm{0}, v_0 \!=\! \bm{0}$, $A_0 \!=\! 0$, $0 \!\leq \mu \!< 1$, $0 \!\leq\! \lambda \!\leq\! 1/(1\!-\!\mu)$, and weights $\{a_t\}$.
   \FOR{$t= 1,2,\ldots,T$}
    \STATE Sample a stochastic gradient $g_t$;
        \FOR {$k=1,2,\ldots,d$}
        \STATE $v_{t,k} = v_{t-1,k} + a_t g_{t,k}^2$;
        \STATE $A_t = A_{t-1} + a_t$;
        \STATE $\bar{a}_t = A_t / t$;
        \STATE $m_{t,k} = \mu m_{t-1,k} - \eta g_{t,k}/(\sqrt{v_{t,k}/\bar{a}_t}+\epsilon)$;
        \STATE $ x_{t+1,k} = x_{t,k} + m_{t,k} + \lambda \mu (m_{t,k} - m_{t-1,k})$;
        \ENDFOR
   \ENDFOR
 \end{algorithmic}
 \end{algorithm} 
 \vspace{-0.2cm}

 Note that AdaUSM extends the AccAdaGrad in \cite{levy2018online} by using more general weighted parameters $a_{t}$ and momentum accelerated mechanisms in the non-convex stochastic setting. 
 In addition, when $\lambda = 0$, AdaUSM reduces to the weighted AdaGrad with a heavy ball momentum, which we denote by {\bf AdaHB} for short. 
 When $\lambda = 1$, AdaUSM reduces to the weighted AdaGrad with Nesterov accelerated  gradient momentum, which we denote by {\bf AdaNAG} for short. 

\begin{algorithm}[H]
\caption{\ AdaHB:  AdaGrad with HB}
\label{Alg:AdaHB}
\begin{algorithmic}[1]
    \STATE {\bf Parameters:} Choose $x_1 \!\in\! \mathbb{R}^d$, fixed parameter $\eta \ge 0$, momentum factor $\mu$, initial accumulator value $\epsilon \ge 0$, and parameters $\{a_t\}$. Set $m_0 = \bm{0},\ v_0 = \bm{0}$, and $A_0 = 0$.
   \FOR{$t= 1,2,\ldots,T$}
    \STATE $A_t = A_{t-1} + a_t$;
    \STATE $\bar{a}_t = A_t / t$;
    \STATE Sample a stochastic gradient $g_t$;
        \FOR {$k=1,2,\ldots,d$}
        \STATE $v_{t,k} = v_{t-1,k} + a_t g_{t,k}^2$;
        \STATE $m_{t,k} = \mu m_{t-1,k} - \frac{\eta g_{t,k}}{\epsilon + \sqrt{v_{t,k}/{\bar{a}_t}}}$;
        \STATE $ x_{t+1,k} = x_{t,k} + m_{t,k}$;    
        \ENDFOR
   \ENDFOR
 \end{algorithmic}
 \end{algorithm} 

\begin{algorithm}[H]
\caption{\ AdaNAG:  AdaGrad with NAG}
\label{Alg:AdaNAG}
\begin{algorithmic}[1]
    \STATE {\bf Parameters:} Choose $x_1 \!\in\! \mathbb{R}^d$, fixed parameter $\eta \ge 0$, momentum factor $\mu$, initial accumulator value $\epsilon \ge 0$, and parameters $\{a_t\}$. Set  $m_0 = \bm{0},\ v_0 = \bm{0}$, and $A_0 = 0$.
   \FOR{$t= 1,2,\ldots,T$}
    \STATE $A_t = A_{t-1} + a_t$;
    \STATE $\bar{a}_t = A_t / t$;
    \STATE Sample a stochastic gradient $g_t$;
        \FOR {$k=1,2,\ldots,d$}
        \STATE $v_{t,k} = v_{t-1,k} + a_t g_{t,k}^2$;
        \STATE $m_{t,k} = \mu m_{t-1,k} - \frac{\eta g_{t,k}}{\epsilon + \sqrt{v_{t,k}/\bar{a}_t}}$;
        \STATE $ x_{t+1,k} = x_{t,k} + m_{t,k} + \mu (m_{t,k} - m_{t-1,k})$;   
        \ENDFOR
   \ENDFOR
\end{algorithmic}
\end{algorithm}         

\subsection{Convergence Results}
\begin{theorem}\label{AdaUSM-high-probility}
Let $\{x_t\}$ be a sequence generated by AdaUSM. 
Assume that the noisy gradient $g_{t}$ satisfies assumptions (\textbf{A1})- (\textbf{A2}), and the weight $\{a_t\}$ is non-decreasing in $t$. 
Let $\tau$ be uniformly randomly drawn from $\{1, 2, \ldots T\}$. Then
\begin{equation*}
\big(\mathbb{E}\norm{\nabla f(x_\tau)}^{4/3}\big)^{3/2} \leq Bound(T) = \mathcal{O}\big(\frac{d\log\big(\sum_{t=1}^T a_t\big)}{\sqrt{T}}\big),
\end{equation*}
where 
\begin{equation*}
    Bound(T) = \frac{\sqrt{2\epsilon^2 + 2\sigma^2 T}}{\eta T} \left[C_1  +  C_2\log\big(1 +  \frac{\sigma^2}{\epsilon^2}\sum_{t=1}^T a_t\big)\right]
\end{equation*}
with $C_1 =  \frac{2(f(x_1) -f^*)}{(1 +\lambda\mu)\eta}$ and $C_2 =\frac{2\eta(1+2\lambda)^2 L d}{(1+\lambda\mu)(1-\mu)^3}+\frac{4\sigma d}{1-\mu}$, respectively.
\end{theorem}

\begin{remark}
When we take $a_t = t^\alpha$ for some constant power $\alpha \geq 0$, then $\sum_{t=1}^T a_t\!=\!\mathcal{O}(T^{\alpha+1})$ and conditions (\textbf{i})-(\textbf{ii}) are satisfied. Hence, AdaUSM with such weights has the $\mathcal{O}(\log(T)/\sqrt{T})$ convergence rate. 
In fact, AdaUSM is convergent as long as $\log\big(\textstyle\sum_{i=1}^{T}a_{t}\big)\!=\!o(\sqrt{T})$ . 
\end{remark}

When taking interpolation factor $\lambda = 1$, and 
$$ a_t= \left\{
\begin{aligned}
& 1, & t \le 2 \\
& \big((1 + t)/4\big)^2, & t \ge 3
\end{aligned}
\right.
$$
AdaUSM reduces to coordinate-wise AccAdaGrad \cite{levy2017online}. In this case, $\sum_{t=1}^T a_t \!=\! \mathcal{O}(T^{3})$. Thus, we have the following corollary for the convergence rate of AccAdaGrad in the non-convex stochastic setting.
\begin{corollary}\label{AdaGrad-high-probility}
Assume the same setting as Theorem \ref{AdaUSM-high-probility}. 
Let $\tau$ be randomly selected from $\{1, 2,\ldots T\}$ with equal probability $\mathcal{P}(\tau = t) = 1/T$. Then
\begin{equation*} 
\left(\mathbb{E}\left[\norm{\nabla f(x_\tau)}^{4/3}\right]\right)^{{3/2}} \leq Bound(T) = \mathcal{O}\left({d}{\log T}/{\sqrt{T}}\right),
\end{equation*}
where 
\begin{equation}
\small   Bound(T)=\frac{\sqrt{2\epsilon^2+2\sigma^2T}}{\eta T} \left[C_1+C_{2}\log\big(1+ \frac{\sigma^2T^3}{\epsilon^2}\big)\right] 
\end{equation}
and $C_{1} = \frac{2}{(1+\mu)\eta}(f(x_1) - f^*)$ and $C_2 =\frac{18\eta L d}{(1+\mu)(1-\mu)^3}+\frac{4\sigma d}{1-\mu}$, respectively.
\end{corollary}

\begin{remark}
 The $\mathcal{O}({\log(T)}/{T})$ non-asymptotic convergence rate measured by the objective for AccAdaGrad has already been established in \cite{levy2018online} in the convex stochastic setting. 
 Corollary \ref{AdaGrad-high-probility} provides the convergence rate of coordinate-wise AccAdaGrad measured by gradient residual, which also supplements the results in \cite{levy2018online} in the non-convex stochastic setting.
\end{remark}

\begin{corollary}
For any $\delta \in (0,1)$, assume the same setting as Theorem \ref{AdaUSM-high-probility}. 
Let $\tau$ be randomly selected from $\{1, 2,\ldots T\}$ with equal probability $\mathcal{P}(\tau = t) = 1/T$. Then
\begin{equation*} 
\mathcal{P}\big(\norm{\nabla f(x_\tau)}^2 < (Bound(T)/\delta)\big) > 1-\delta^{2/3}.
\end{equation*}
where 
\begin{equation}
\small   Bound(T)=\frac{\sqrt{2\epsilon^2+2\sigma^2T}}{\eta T} \left[C_1+C_{2}\log\big(1+ \frac{\sigma^2T^3}{\epsilon^2}\big)\right] 
\end{equation}
and $C_{1} = \frac{2}{(1+\mu)\eta}(f(x_1) - f^*)$ and $C_2 =\frac{18\eta L d}{(1+\mu)(1-\mu)^3}+\frac{4\sigma d}{1-\mu}$, respectively.
\end{corollary}
\begin{remark}
By Markov's inequality, for any random variable $\xi$ we have 
\begin{equation}
\mathcal{P}(|\xi| \ge \Gamma) \leq \mathbb{E}[|\xi|] / \Gamma,
\end{equation}
where $\mathcal{P}$ denotes the probability.
For any small positive $\delta$, by taking $\xi = \norm{\nabla f(x_\tau)}^{4/3}$ and $\Gamma = (Bound(T)/\delta)^{2/3}$ we have 
\begin{equation}
\begin{split}
    & \mathcal{P}\big(\norm{\nabla f(x_\tau)}^2 \geq (Bound(T)/\delta)\big) \\
    & = \mathcal{P}\big(\norm{\nabla f(x_\tau)}^{4/3} \ge (Bound(T)/\delta)^{2/3}\big) \\ 
    & \leq \mathbb{E}(\norm{\nabla f(x_\tau)}^{4/3}) * (\delta / Bound(T))^{2/3} 
    \leq \delta^{2/3}.
\end{split}
\end{equation}
Hence, $\mathcal{P}\big(\norm{\nabla f(x)}^2 < (Bound(T)/\delta)\big) > 1 - \delta^{2/3}$. The result shows that the square of absolute norm of the gradient converges to zero with high probability.
\end{remark}

\section{Relationships with Adam and RMSProp}\label{relationship}

In this section, we show that the exponential moving average (EMA) technique in estimating adaptive learning rates in Adam \cite{kingma2014adam} and RMSProp \cite{hinton2012lecture} is a special case of the weighted adaptive learning rate in Eq.~\eqref{adaptive-learning-rate}, 
\textit{i.e.}, their adaptive learning rates correspond to taking exponentially growing weights in AdaUSM, which thereby provides a new angle for understanding Adam and RMSProp.

\subsection{Adam}\label{Adam-momentum}
For better comparison, we first represent the $t$-th iterate scheme of Adam \cite{kingma2014adam} as follows:
\begin{equation*}
\left\{
\begin{aligned}
& \widehat{m}_{t,k} = \beta_{1} \widehat{m}_{t-1,k} + (1-\beta_{1}) g_{t,k},\\
& m_{t,k} = \widehat{m}_{t,k}/(1-\beta_{1}^{t}), \\
& \widehat{v}_{t,k} = \beta_2 \widehat{v}_{t-1,k} +  (1-\beta_2)g_{t,k}^2, \\
&  v_{t,k} = \widehat{v}_{t,k}/(1-\beta_{2}^{t}), \\
& x_{t,k} =  x_{t-1,k} - \eta m_{t,k}/(\sqrt{t v_{t,k}}+\sqrt{t}\epsilon),
\end{aligned}
\right.
\end{equation*}
for $k = 1, \ldots,d$, where $\beta_{1}, \beta_{2} \in [0,1)$ and $\eta > 0$ are constants, and $\epsilon$ is a sufficiently small constant. 
Denoting $\eta_{t,k} = \eta /(\sqrt{t v_{t,k}}+t\epsilon)$, we can simplify the iterations of Adam as
\begin{equation} \label{simplify-adam-no}
 \left\{
\begin{aligned}
& \widehat{m}_{t,k} = \beta_{1} \widehat{m}_{t-1,k} + (1-\beta_{1}) g_{t,k}, \\
& m_{t,k} = \widehat{m}_{t,k}/(1-\beta_{1}^{t}), \\
& x_{t+1,k} = x_{t,k} - \eta_{t,k} m_{t,k}.
\end{aligned}
\right.
\end{equation} 
Below, we show that AdaUSM and Adam differ in two aspects: momentum estimation $m_{t,k}$ and coordinate-wise adaptive learning rate $\eta_{t,k}$.

{\bf Momentum estimation.}\ ~The EMA technique is widely used in the momentum estimation step in Adam \cite{kingma2014adam} and AMSGrad \cite{reddi2018convergence}. Without loss of generality\footnote{Note that an extra step $m_{t,k} = \widehat{m}_{t,k}/(1-\beta_{1}^{t})$ is introduced in Adam to correct the bias. However, the bias-correction step does not affect the global convergence \cite{reddi2018convergence,chen2018convergence}.},
we consider the simplified EMA step 
\begin{equation}\label{simplify-adam}
 \left\{
\begin{aligned}
& m_{t,k} = \beta_{1} m_{t-1,k} + (1-\beta_{1}) g_{t,k}, \\
& x_{t+1,k} = x_{t,k} - \eta_{t,k} m_{t,k}.
\end{aligned}
\right.
\end{equation} 
To show the difference clearly, we merely compare HB momentum with EMA momentum. Let $\widetilde{m}_{t,k}  =  -\eta_{t,k}m_{t,k}$. By the first equality in Eq. \eqref{simplify-adam}, we have
\begin{align}\label{EMA-HB}
\begin{split}
\widetilde{m}_{t,k} 
& = \beta_{1}\eta_{t,k} m_{t-1,k} - (1-\beta_1)\eta_{t,k} g_{t,k} \\
& = \beta_{1}\widetilde{m}_{t-1,k} - (1 - \beta_1)\eta_{t,k} g_{t,k} \\
&\qquad + (\eta_{t,k} - \eta_{t-1,k})m_{t-1,k}.
\end{split}
\end{align}
The EMA momentum corresponds to the HB momentum with $(1-\beta_1)\eta_{t,k}$ stepsize and an extra error term $(\eta_{t,k} -\eta_{t-1,k})m_{t-1,k}$. The error term vanishes if the stepsize $\eta_{t,k}$ is taking constant\footnote{Since  both the learning rates $\eta_{t,k}$ in AdaHB and Adam are determined adaptively, we do not have $\eta_{t,k} = \eta_{t-1,k}$.}. More precisely, if we write the iterate $x_{t,k}$ in terms of stochastic gradients $\{g_{1,k},g_{2,k}, \ldots, g_{t,k}\}$ and eliminate $m_{t,k}$, we obtain
\begin{equation*}
 \left\{
\begin{aligned}
& {\rm AdaUSM:\ } x_{t+1,k} = x_{t,k}-\textstyle\sum_{i=1}^t \eta_{i,k} g_{i,k} \beta_1^{t-i}, \\
& {\rm EMA:\ } x_{t+1,k} = x_{t,k} - (1-\beta_1)\eta_{t,k}\textstyle\sum_{i=1}^t g_{i,k} \beta_1^{t-i}.
\end{aligned}
\right.
\end{equation*} 
One can clearly see that the key difference lies in whether we use the current stepsize or the past stepsize in the exponential moving average. 
One can see that AdaUSM uses the past step-sizes but EMA uses only the current one in exponential moving averaging.
Moreover, when momentum factor  $\beta_{1}$ is very close to $1$, the update of $x_{k}$ via EMA could stagnate since $\|x_{k+1}-x_{k}\|=\mathcal{O}(1-\beta_1)$. This dilemma will not appear in AdaUSM. 

{\bf Adaptive learning rate.}\ ~Note that $\widehat{v}_{t,k} =  \beta_2 \widehat{v}_{t-1,k} +  (1 - \beta_2)g_{t,k}^2$. We have $\widehat{v}_{t,k} =  \beta_2^t\widehat{v}_{0,k}  +  \textstyle\sum_{i=1}^{t}{(1 - \beta_2)}\beta_2^{t-i}g^2_{i,k}$. Without loss of generality, we set $\widehat{v}_{0,k} = 0$. Hence, it holds that
\[
{v}_{t,k} = \widehat{v}_{t,k}/(1-\beta_{2}^{t}) = \sum_{i=1}^{t}\frac{(1-\beta_2)}{1-\beta_{2}^{t}}\beta_2^{t-i}g^2_{i,k}.
\] 
Then, the learning rate $\eta_{t,k}$ in Adam can be rewritten as 
\begin{align}\label{adam-learning-rate}
\begin{split}
 \eta_{t,k} 
& =  \frac{\eta }{\sqrt{t v_{t,k}}+t\epsilon} \\
& = \frac{\eta}{\sqrt{t\sum_{i=1}^{t}\frac{(1-\beta_2)}{1-\beta_2^{t}}\beta_2^{t-i}g^2_{i,k}}+t\epsilon} \\
& = \frac{\eta}{\sqrt{t\sum_{i=1}^{t}\frac{(1-\beta_2)\beta_2^{t}}{1-\beta_2^{t}}\beta_2^{-i}g^2_{i,k}}+t\epsilon}.
\end{split}
\end{align}
Let $a_i = \beta_2^{-i}$. Note that 
$ \sum_{i=1}^t \beta_2^{-i} = \frac{1-\beta_2^t}{(1-\beta_2)\beta_2^t}.$
Hence, Eq.~\eqref{adam-learning-rate} can be further reformulated as
\begin{align} 
\eta_{t,k}= \frac{\eta}{\sqrt{t\sum_{i=1}^{t}\frac{a_{i}}{\sum_{i=1}^{t}a_{i}}g^2_{i,k}} + \sqrt{t}\epsilon}.
\end{align}
For comparison, the adaptive learning rates of Adam and AdaUSM are summarized as follows:
\begin{equation*}
 \left\{
\begin{aligned}
&  \eta_{t,k}^{\tiny \rm  Adam} = {\eta}{\Big /}\left[\textstyle \sqrt{t\sum_{i=1}^{t}\frac{a_{i}}{\sum_{i=1}^{t}a_{i}}g^2_{i,k}}+\sqrt{t}\epsilon\right], \\
&  \eta_{t,k}^{\tiny \rm AdaUSM} ={\eta}{\Big /}\left[\textstyle\sqrt{t\sum_{i=1}^{t}\frac{a_{i}}{\sum_{i=1}^{t}a_{i}}g^2_{i,k}}+\epsilon\right].
\end{aligned}
\right.
\end{equation*} 
Hence, the adaptive learning rate in Adam is actually equivalent to that in AdaUSM by specifying $a_{i} = \beta_{2}^{-i}$ if $\epsilon$ is sufficiently small. For the parameter setting in Adam, it holds that  
\begin{align}
&\log\big(\sum_{i=1}^{T}a_{i}\big)= \log\big(\frac{1-\beta_2^{-T}}{1-\beta_2^{-1}}\big) \ge \log\big(\frac{\beta_2^{-T+1}}{\beta_2^{-1}-1}\big) \nonumber\\
&\quad \ge (T-1)\log(\frac{1}{1-\beta}) =  \mathcal{O}(T)  > o(T). \nonumber
\end{align}
Thus, we gain an insight for understanding the convergence of Adam from the convergence results of AdaUSM in Theorem \ref{AdaUSM-high-probility}.

\begin{remark}
Recently, Chen \textit{et al.} \cite{chen2018convergence} have also proposed AdaGrad with exponential moving average (AdaEMA) by setting $\beta_{2}$ as $(\beta_{2})_{t}= 1-{1}/{t}$ and removing the bias-correction steps in Adam. 
Its $\mathcal{O}({\log{(T)}}/{T})$ convergence rate in the non-convex stochastic setting has been established under a slightly stronger assumption that the stochastic gradient estimate is required to be uniformly bounded.
Compared with AdaEMA, AdaUSM not only adopts a general weighted sequence in estimating the adaptive learning rate, but also uses a different unified momentum that covers HB and NAG as special instances. 
The superiority of HB and NAG momentums over EMA has been pointed out in the above paragraph: {\bf Momentum estimation}. In Section \ref{experiment}, we also experimentally demonstrate the effectiveness of AdaUSM against AdaEMA.  
\end{remark}

\subsection{RMSProp}  
Coordinate-wise RMSProp is another efficient solver for training DNNs~\cite{hinton2012lecture,mukkamala2017variants}, which is defined as
\begin{equation*} 
\left\{
\begin{aligned}
& v_{t,k} = \beta v_{t-1,k} +  (1-\beta)g_{t,k}^2 \\
& x_{t,k} =  x_{t-1,k} - \eta g_{t,k}/(\sqrt{tv_{t,k}}+\epsilon)
\end{aligned}
\right., {\rm\ for\ } k = 1, \ldots, d.
\end{equation*}
Define $a_i = \beta^{-i}$. The adaptive learning rate of RMSProp denoted as $\eta_{t,k}^{\tiny \rm RMSProp}$ can be rewritten as 
\begin{align*}
 \eta_{t,k}^{\tiny \rm RMSProp} 
& =  \frac{\eta }{\sqrt{t v_{t,k}} + \epsilon} \\
& = \frac{\eta/(1 - \beta^{t})}{\sqrt{t\sum_{i=1}^{t}\frac{(1-\beta_2)}{1-\beta_2^{t}}\beta_2^{t-i}g^2_{i,k}} + \frac{\epsilon}{1 - \beta^{t}}} \\
& =\frac{\eta/(1 - \beta^{t})}{\textstyle\sqrt{t\sum_{i=1}^{t}\frac{a_{i}}{\sum_{i=1}^{t}a_{i}}g^2_{i,k}}+\frac{\epsilon}{1-\beta^{t}}}.
\end{align*}
When $\epsilon$ is a sufficiently small constant and $\beta <1$, it is obvious that $\eta_{t,k}^{\tiny \rm RMSProp}$ has a similar structure to $\eta_{t,k}^{\tiny \rm AdaUSM}$ after $t$ being sufficiently large. Based on the above analysis, AdaUSM can be interpreted as generalized RMSProp with HB and NAG momentums.

\begin{figure*}[htpb]
\vspace{-0.35cm}
\centering
\subfigure{\includegraphics[width=0.25\linewidth,height=0.25\linewidth]{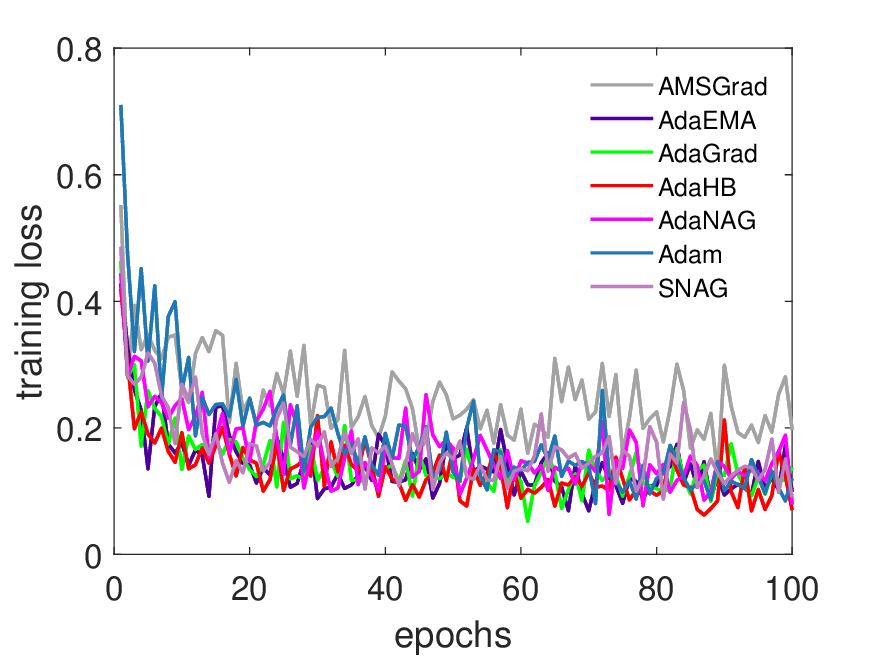}} \!\!\!\!\!\!  
\subfigure{\includegraphics[width=0.25\linewidth,height=0.25\linewidth]{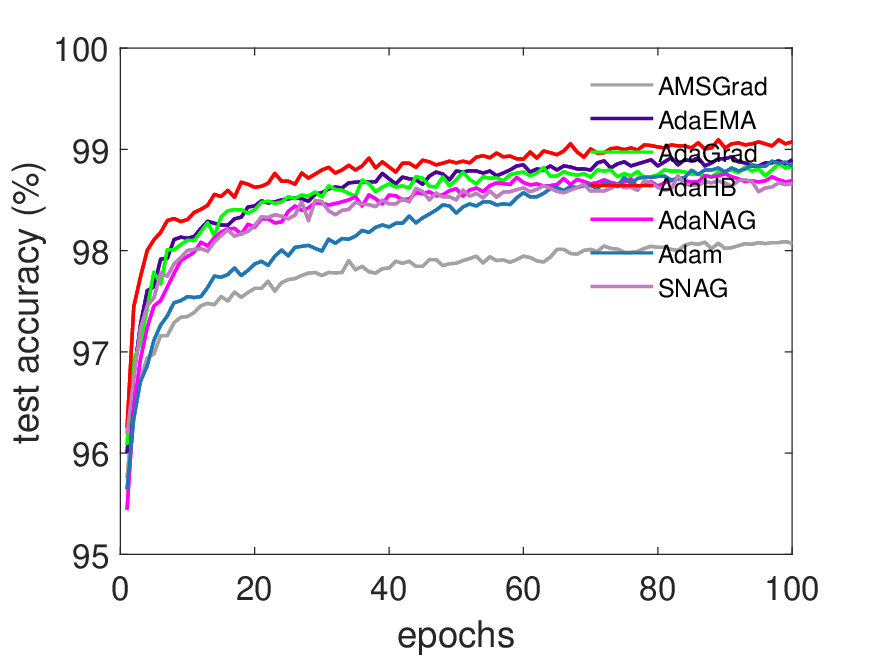}} \!\!\!\!\!\!  
\subfigure{\includegraphics[width=0.25\linewidth,height=0.25\linewidth]{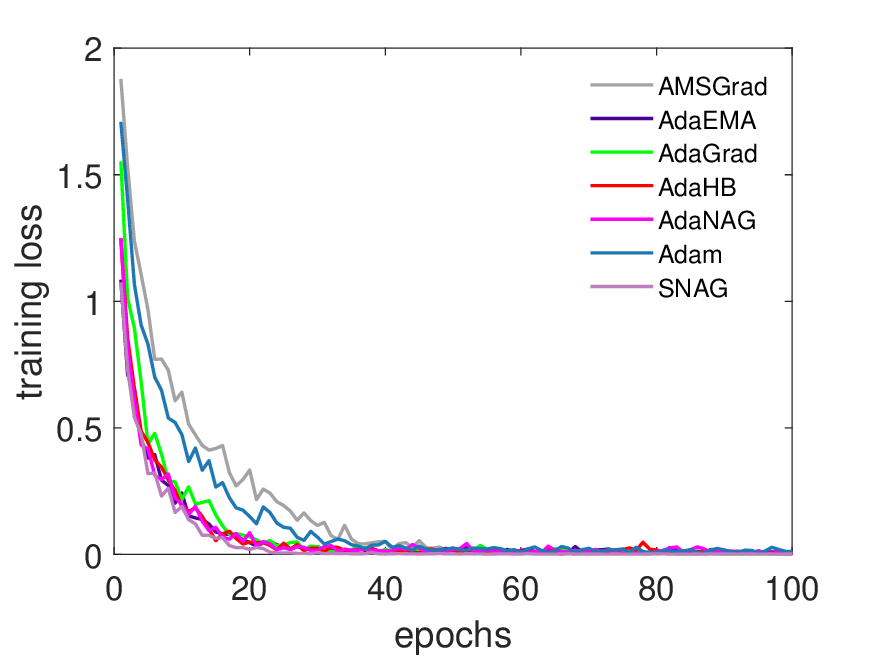}}\!\!\!\!\!\!  
\subfigure{\includegraphics[width=0.25\linewidth,height=0.25\linewidth]{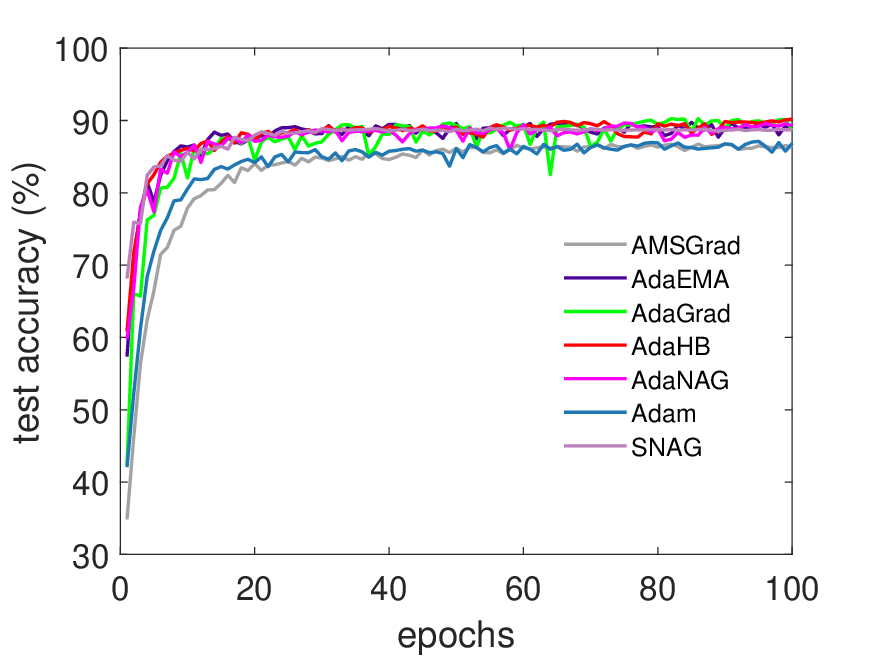}}\!\!\!\!\!\!   
\vspace{-0.35cm}
\caption{The first two and the last two figures illustrate the performance profiles of various optimizers for training LeNet on MNIST and training GoogLeNet on CIFAR10, respectively. 
}
\label{fig:LeNet}
\centering
\subfigure{\includegraphics[width=0.25\linewidth,height=0.25\linewidth]{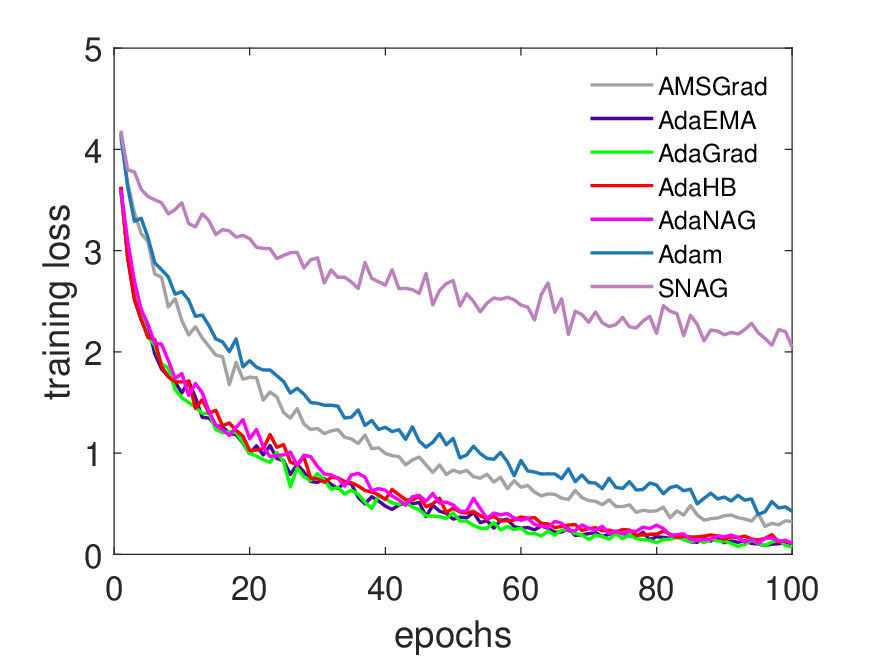}}\!\!\!\!\!\!  
\subfigure{\includegraphics[width=0.25\linewidth,height=0.25\linewidth]{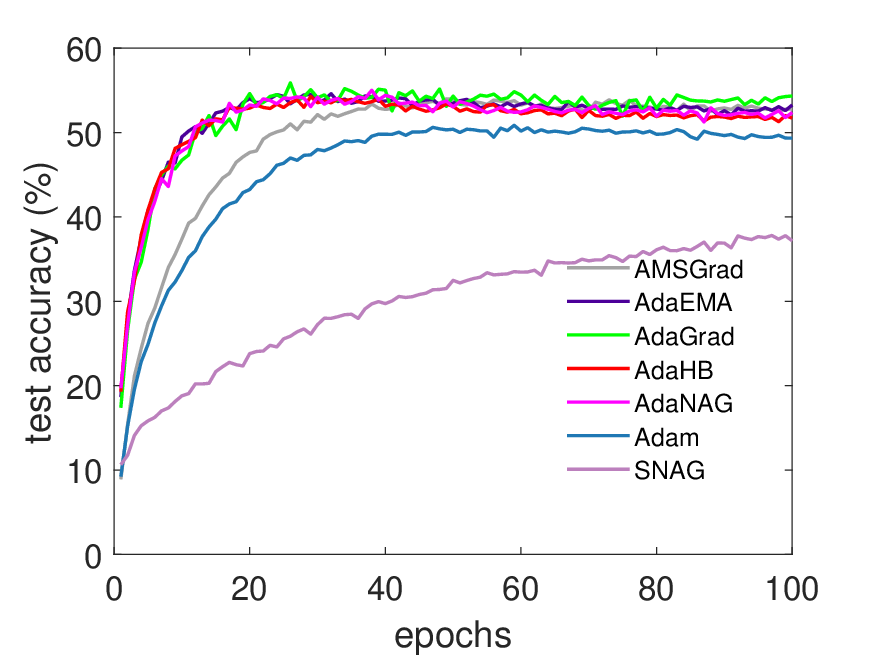}}\!\!\!\!\!\! 
\subfigure{\includegraphics[width=0.25\linewidth,height=0.25\linewidth]{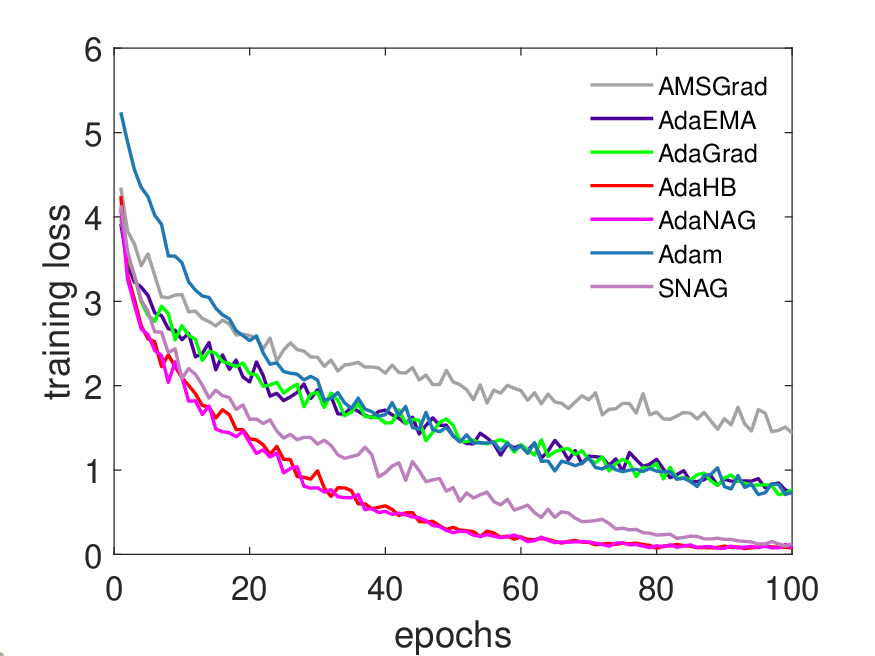}}\!\!\!\!\!\!  
\subfigure{\includegraphics[width=0.25\linewidth,height=0.25\linewidth]{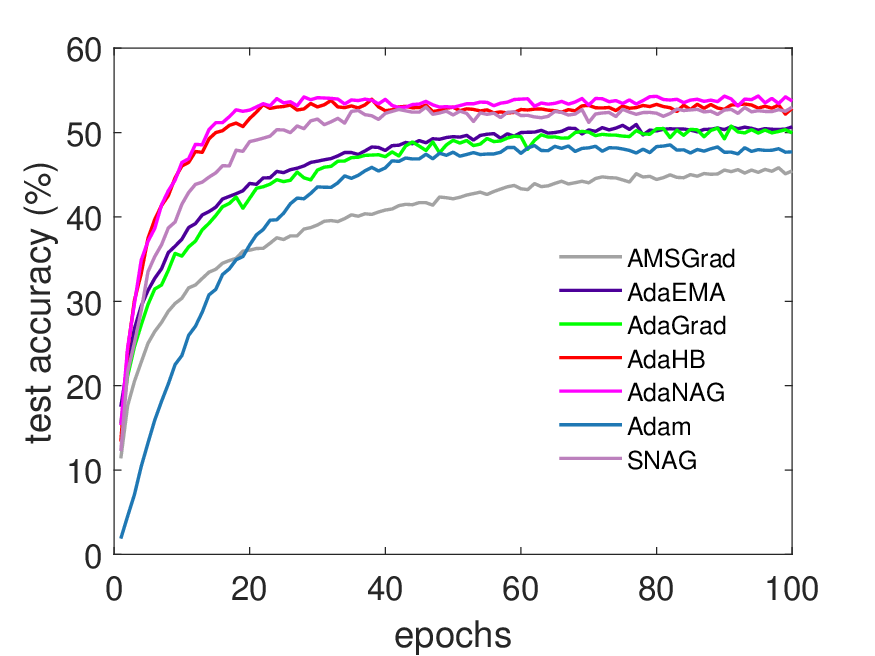}}\!\!\!\!\!\!  
\vspace{-0.35cm}
\caption{The first two and the last two figures illustrate the Performance profiles of various optimizers for training DenseNet on CIFAR100 and training ResNet on Tiny-ImageNet, respecgively. 
}
\label{fig:DenseNet}
\label{fig:ResNet}
\vspace{-0.35cm}
\end{figure*}

\section{Experiments}\label{experiment}

In this section, we conduct experiments to validate the efficacy and theory of AdaHB (AdaUSM with $\lambda = 0$) and AdaNAG (AdaUSM with $\lambda = 1$) by applying them to train DNNs\footnote{https://github.com/kuangliu/pytorch-cifar} including LeNet \cite{lecun1998gradient}, GoogLeNet \cite{szegedy2015going}, ResNet \cite{he2016deep}, and DenseNet \cite{huang2017densely} on various datasets including MNIST \cite{lecun1998gradient}, CIFAR10/100 \cite{krizhevsky2009learning}, and Tiny-ImageNet \cite{deng2009imagenet}. 
The efficacies of AdaHB and AdaNAG are evaluated in terms of the training loss and test accuracy v.s. epochs, respectively. In the experiments, we fix the batch-size as $128$ and the weighted decay parameter as $5 \times 10^{-4}$, respectively. 

{\bf Optimizers.}\quad~ We compare AdaHB/AdaNAG with five competing algorithms: SGD with momentum (SGD-momentum) \cite{sutskever2013importance}, AdaGrad \cite{duchi2011adaptive,ward2018adagrad}, AdaEMA \cite{chen2018convergence}, AMSGrad \cite{reddi2018convergence,chen2018convergence}, and Adam \cite{kingma2014adam}. The parameter settings of all compared optimizers are summarized in Table \ref{paramter-comparaed-1}.

\textit{{Because the main contribution is to give a theoretical analysis for the AdaUSM, we conduct experiments based on theoretical settings}. To match the convergence theory, we take the diminishing base learning rate as ${\eta}/{\sqrt{t}}$ uniformly across all the tested adaptive optimizers.}
Moreover, via the momentum estimation paragraph in Section \ref{Adam-momentum}, we know that the learning rates in AdaHB and AdaNAG will be ${1}/(1 - \beta_{1})$ times greater than those in AdaEMA, AMSGrad, and Adam if they share the same constant parameter $\eta$. 
In addition, too large and small $\eta$ would lead to heavy oscillation and bad stagnation on the training loss, respectively, which would deteriorate the performances of the tested optimizers. Consequently, the base learning rate $\eta$ for each solver is chosen via grid search on the set $\{1,\, 0.1,\, 0.01,\, 0.001,\, 0.0001\}$. We report the base learning rate of each solver that can consistently contribute to the best performance.

\begin{table}[H]
\caption{Parameter settings of AdaHB, AdaNAG, AdaGrad, SGD-momentum, AdaEMA, AMSGrad, and Adam.}
\label{paramter-comparaed-1}
\begin{center}
\begin{tabular}{|c|c|c|c|c|}
\hline
Name     & $\eta$ & $\beta_{1}$ & $\beta_2$ & $\epsilon$ \\ \hline
AdaEMA   & 0.01  &   0.9      & $1-1/t$   & 1.0e-8     \\ 
AMSGrad  & 0.01  &   0.9      & 0.999     & 1.0e-8     \\ 
Adam     & 0.01  &   0.9      & 0.999     & 1.0e-8     \\ 
\hline
\end{tabular}

\bigskip

\begin{tabular}{|c|c|c|c|c|}
\hline
Name   & $\eta$ & weights $a_{t}$ & $\mu$ &  $\epsilon $ \\ \hline
SGD-momentum   & 0.1  &  $\diagup$      &  0.9  &   $\diagup$  \\ 
AdaGrad& 0.01 &     1           &  0.9  &   1.0e-8     \\  
AdaHB  & 0.001 &     t           &  0.9  &   1.0e-8     \\
AdaNAG & 0.001 &     t           &  0.9  &   1.0e-8     \\
\hline
\end{tabular}
\end{center}
\end{table}

{\bf Experimental results.} We conduct numerical experiments to compare the performances of AdaHB, AdaNAG, SGD-momentum, AdaGrad, AdaEMA, AMSGrad, and Adam by applying them to training LeNet on MNIST, GoogLeNet on CIFAR10, DenseNet on CIFAR100, and ResNet on Tiny-ImageNet, respectively. The results are shown in Figures \ref{fig:LeNet}-\,\ref{fig:ResNet}.

The first two figures in Figure \ref{fig:LeNet} illustrates the performance of LeNet on MNIST which covers $60,000$ training examples and $10,000$ test examples. It can be seen that AdaHB decreases the training loss fastest among the seven tested  optimizers, which simultaneously yields a higher test accuracy than the other tested optimizers. The performances of AdaEMA and AdaGrad are worse than AdaHB but better than SGD-momentum, Adam, and AMSGrad.     

The last two figures in Figure \ref{fig:LeNet}  illustrate the performance of training GoogLeNet on CIFAR10 which covers $50,000$ training examples and $10,000$ test examples. 
It can be seen that SGD-momentum decreases the training loss slightly faster than other optimizers, followed by AdaHB and AdaNAG, and that AMSGrad and Adam are the slowest optimizers. 
The test accuracies of AdaHB, AdaNAG, AdaGrad, and AdaEMA are comparable, which are all slightly better than SGD-momentum and outperform Adam and AMSGrad. 

The first two figures in Figure \ref{fig:ResNet} illustrates the performance of training DenseNet on CIFAR100 which covers $50,000$ training examples and $10,000$ test examples. 
It shows that SGD-momentum has the worst training process and test accuracy, followed by Adam and AMSGrad. While AdaGrad, AdaEMA, AdaHB, and AdaNAG decrease the training loss and increase the test accuracy at roughly the same speed. 

The last two figures in Figure \ref{fig:ResNet} illustrate the performance of training ResNet on Tiny-ImageNet which contains $100,000 $ training examples and $10,000$ test examples. It can be seen that AdaHB and AdaNAG show the fastest speed to decrease the training loss and increase the test accuracy. SGD-momentum is worse than AdaHB and AdaNAG, and better than AdaGrad, AdaEMA, and AMSGrad in terms of the training loss and test accuracy. 

In summary, AdaHB and AdaNAG are more efficient and robust than SGD-momentum, AdaGrad, AdaEMA, Adam, and AMSGrad in terms of both the training speed and generalization capacity. SGD-momentum is also an efficient optimizer but it is highly sensitive to the hand-tuning learning rate. Moreover, the value of EMA is marginal and not as efficient as heavy ball and Nesterov accelerated gradient momentums, as revealed by the performance curves of AdaEMA, AdaGrad, AdaHB, and AdaNAG. 

\section{Conclusions}

We integrated a novel weighted coordinate-wise AdaGrad with unified momentum including heavy ball and Nesterov accelerated gradient momentums, yielding a new adaptive stochastic algorithm called AdaUSM. 
Its $\mathcal{O}({\log(T)}/{\sqrt{T}})$ convergence rate was established in the non-convex stochastic setting. Our work largely extends the convergence rate of accelerated AdaGrad in \cite{levy2018online} to the general non-convex stochastic setting.
Moreover, we pointed out that the adaptive learning rates of Adam and RMSProp are essentially special cases of the weighted adaptive learning rate in AdaUSM, which provides a new angle to understand the convergences of Adam/RMSProp. 
We also experimentally verified the efficacy of AdaUSM in training deep learning models on several image datasets. 
The promising results show that the proposed AdaUSM is more effective and robust than SGD with momentum, AdaGrad, AdaEMA, AMSGrad, and Adam in terms of the training loss and test accuracy v.s. epochs.  

\nocite{szegedy2015going}
\nocite{nguyen2017sarah}
\nocite{zhuang2019surrogate}
\nocite{agarwal2019efficient}
\nocite{liu2019variance}
\nocite{kavis2019unixgrad}
\nocite{chen2019zo}

\bibliography{bibtex}
\bibliographystyle{IEEEtran}

\clearpage
\appendix 
\section{Detailed Proof}
In this section, we give a complete proof of Theorem \ref{AdaUSM-high-probility}.
The section is arranged as follows. In Section \ref{Preliminary lemmas} we provide preliminary lemmas that will be used to establish Theorem \ref{AdaUSM-high-probility}. In Section \ref{main proof} we give the detailed proof of Theorem \ref{AdaUSM-high-probility}. 

\subsection{Preliminary Lemmas}\label{Preliminary lemmas}
First, in this section we provide preliminary lemmas that will be used to prove our main theorem. The readers may skip this part for the first time and come back whenever the lemmas are needed.
\begin{lemma}\label{lem1}
Let $S_t = S_0 + \sum_{i=1}^t a_i$, where $\{a_t\}$ is a non-negative sequence and $S_0 > 0$. We have 
$
\sum_{t=1}^T \frac{a_t}{S_t} \leq \log(S_T) - \log(S_0).
$
\end{lemma}
\begin{proof}
The finite sum $\sum_{t=1}^T \frac{a_t}{S_t}$ can be interpreted as a Riemann sum as follows $\sum_{t=1}^T \frac{1}{S_t}(S_t - S_{t-1}).$
Since $1/x$ is decreasing on the interval $(0, \infty)$, we have 
$$\sum_{t=1}^T \frac{1}{S_t}(S_t-S_{t-1}) \leq \int_{S_0}^{S_T} \frac{1}{x} d x = \log(S_T) - \log(S_0).$$
The proof is finished.
\end{proof}

The following lemma is a direct result of the momentum updating rule.
\begin{lemma}\label{lem2}
Suppose $m_t = \mu m_{t-1} - \eta_t g_t$ with $m_0 = \bm{0}$ and $0\leq \mu < 1$. We have the following estimate
\begin{equation}
\sum_{t=1}^T \norm{m_t}^2 \leq \frac{1}{(1-\mu)^2} \sum_{t=1}^T\norm{\eta_t g_t}^2.
\end{equation}
\end{lemma}
\begin{proof}
First, we have the following inequality due to convexity of $\norm{\cdot}^2$:
\begin{equation}\label{eq3016}
\begin{aligned}
    \norm{m_t}^2 &= \norm{\mu m_{t-1} + (1-\mu)(-\eta_t g_t/(1-\mu))}^2 \\
    &\leq \mu\norm{m_{t-1}}^2 + (1-\mu)\norm{\eta_t g_t/(1-\mu)}^2 \\
    &= \mu \norm{m_{t-1}}^2 + \norm{\eta_t g_t}^2/(1-\mu).
\end{aligned}
\end{equation}
Taking sum of Eq.~\eqref{eq3016} from $t=1$ to $t=T$ and using $m_0 = \bm{0}$, we have that
\begin{equation}
\begin{split}
\sum_{t=1}^T \norm{m_t}^2 
&\leq \mu\sum_{t=1}^{T-1} \norm{m_t}^2 + \frac{1}{1-\mu}\sum_{t=1}^T \norm{\eta_t g_t}^2 \\
&\leq \mu\sum_{t=1}^{T} \norm{m_t}^2 + \frac{1}{1-\mu}\sum_{t=1}^T \norm{\eta_t g_t}^2. 
\end{split}
\end{equation}
Hence,
\begin{equation}
    \sum_{t=1}^T \norm{m_t}^2 \leq \frac{1}{(1-\mu)^2}\sum_{t=1}^T \norm{\eta_t g_t}^2.
\end{equation}
The proof is finished.
\end{proof}

The following lemma is a result of the USM formulation for any general adaptive learning rate.
\begin{lemma}\label{lem3}
Let $\{x_t\}$ and $\{m_t\}$ be sequences generated by the following general SGD with USM momentum:
starting from initial values $x_1$ and $m_0 = 0$, and being updated through
\begin{equation*}
\left\{
\begin{aligned}
m_t &= \mu m_{t-1} - \eta_t g_t, \\
x_{t+1} &= x_t + m_t + \lambda \mu (m_t - m_{t-1}),
\end{aligned}
\right.
\end{equation*}
where the momentum factor $\mu$ and the interpolation factor $\lambda$ satisfy $0 \leq \mu < 1$ and $0 \leq \lambda \leq 1/(1-\mu)$, respectively. 
Suppose that the function $f$ is $L$-smooth. Then for any $t \geq 2$ we have the following estimate
\begin{equation}\label{2-007}
\begin{split}
\langle \nabla f(x_t), m_t \rangle \leq \mu \langle \nabla f(x_{t-1}), m_{t-1} \rangle + 
(1+\frac{3}{2}\lambda \mu) \mu L \norm{m_{t-1}}^2 \\ 
+ \frac{1}{2}\lambda \mu^2 L \norm{m_{t-2}}^2 - \langle \nabla f(x_t), \eta_t g_t \rangle.
\end{split}
\end{equation}
In particular, the following estimate holds
\begin{align}\label{2-008}
\langle \nabla f(x_t), m_t \rangle 
&\leq (1+ 2\lambda)L \sum_{i=1}^{t-1} \norm{m_i}^2 \mu^{t-i} \nonumber\\
&- 
\sum_{i=1}^t \langle \nabla f(x_i), \eta_i g_i \rangle \mu^{t-i}.
\end{align}
\end{lemma}

\begin{proof}
Since $m_t = \mu m_{t-1} - \eta_t g_t$, we have
\begin{equation}\label{2-009}
\begin{split}
&\langle \nabla f(x_t), m_t \rangle \\
& = \mu \langle \nabla f(x_t), m_{t-1} \rangle - \langle \nabla f(x_t), \eta_t g_t \rangle \\
& = \mu \langle \nabla f(x_{t-1}), m_{t-1} \rangle + \mu \langle \nabla f(x_t) - \nabla f(x_{t-1}), m_{t-1} \rangle \\
& \quad - \langle \nabla f(x_t), \eta_t g_t \rangle.
\end{split}
\end{equation}
Note that by $L$-smoothness of function $f$, we have that
\begin{equation}\label{2-010}
\begin{split}
\norm{\nabla f(x_t) - \nabla f(x_{t-1})} 
& \leq L \norm{x_t - x_{t-1}} \\
&= L\norm{m_{t-1} + \lambda \mu (m_{t-1} - m_{t-2})} \\
& \leq (1+\lambda \mu) L \norm{m_{t-1}} + \lambda \mu L \norm{m_{t-2}}.
\end{split}
\end{equation}
Hence, by Cauchy-Schwartz inequality and Eq.~\eqref{2-010}, we have that
\begin{equation}\label{2-011}
\begin{split}
&\langle \nabla f(x_t) - \nabla f(x_{t-1}), m_{t-1} \rangle\\
& \leq \norm{\nabla f(x_t) - \nabla f(x_{t-1})} \norm{m_{t-1}} \\
& \leq (1 + \lambda \mu) L \norm{m_{t-1}}^2 + \lambda \mu L \norm{m_{t-2}}\norm{m_{t-1}} \\
& \leq (1 + \frac{3}{2} \lambda \mu) L \norm{m_{t-1}}^2 + \frac{1}{2}\lambda \mu L \norm{m_{t-2}}^2.
\end{split}
\end{equation}
Combining Eq.~\eqref{2-009} and Eq.~\eqref{2-011}, we obtain the desired inequality in Eq.~\eqref{2-007}.

To obtain the second estimate, let $B_t = \langle \nabla f(x_t), m_t \rangle$. If $\mu = 0$, the equality holds trivially. Otherwise $0 < \mu < 1$. We divide $\mu^t$ from both sides of Eq.~\eqref{2-007} and obtain
\begin{align}\label{2-012}
\frac{B_t}{\mu^t} 
&\leq \frac{B_{t-1}}{\mu^{t-1}} 
+ (1+\frac{3}{2}\lambda \mu) L \frac{\norm{m_{t-1}}^2}{\mu^{t-1}} \nonumber \\ 
&+ \frac{1}{2}\lambda L \frac{\norm{m_{t-2}}^2}{\mu^{t-2}} 
- \langle \nabla f(x_t), \eta_t g_t \rangle \mu^{-t}.
\end{align}
Note that $m_0 = \bm{0}$, and $B_1 = -\langle \nabla f(x_1), \eta_1 g_1\rangle$. Therefore, 
\begin{equation}\label{2-013}
\begin{split}
\frac{B_t}{\mu^t} 
& \leq \frac{B_1}{\mu} + (1 + \frac{3}{2}\lambda \mu)L \sum_{i=2}^{t} \norm{m_{i-1}}^2 \mu^{-(i-1)} \\
& \quad + \frac{1}{2} \lambda L \sum_{i=2}^{t} \norm{m_{i-2}}^2 \mu^{-(i-2)} - \sum_{i=2}^t \langle \nabla f(x_i), \eta_i g_i\rangle \mu^{-i} \\
& \leq (1+2\lambda) L \sum_{i=1}^{t-1} \norm{m_i}^2 \mu^{-i} 
  - \sum_{i=1}^t \langle \nabla f(x_i), \eta_i g_i \rangle \mu^{-i}.
\end{split}
\end{equation}
Multiplying both sides of Eq.~\eqref{2-013} by $\mu^t$, we obtain the desired estimate Eq.~\eqref{2-008}. The proof is completed. 
\end{proof}

The following two lemmas, which are first introduced in \cite[Theorem 10]{ward2018adagrad}, are particularly due to the AdaGrad adaptive learning rate. 
Here we adjust their proofs to the coordinate-wise version for our Weighed AdaGrad adaptive learning rate and represent it here for readers' convenience.

\begin{lemma}\label{lem4}
Let $\sigma_t = \sqrt{\mathbb{E}_t g_t^2}$ and let 
$$\hat{\eta}_t =\frac{\eta}{\sqrt{\frac{1}{\bar{a}_t}\left(\sum_{i=1}^{t-1} a_i g_i^2 + a_t\sigma_t^2\right)}+\bm{\epsilon}}.$$ 
Assume that the noisy gradients $g_t$ satisfy assumptions (\textbf{A1}) and (\textbf{A2}). Then we have the following estimate 
\begin{align}
&-\mathbb{E}_t\langle \nabla f(x_t), \eta_t g_t \rangle \\
&\leq -\frac{1}{2} \norm{\nabla f(x_t)}^2_{\hat{\eta}_t} + \frac{2\sigma}{\eta} \mathbb{E}_t \left[\sqrt{(a_t/\bar{a}_t)}\norm{\eta_t g_t}^2\right], 
\end{align}
where $\norm{\nabla f(x_t)}^2_{\hat{\eta}_t} = \sum_{k=1}^d \hat{\eta}_{t, k} |\nabla_k f(x_t)|^2$, and the constant $\sigma$ is defined in assumption (\textbf{A2}).
\end{lemma}

\begin{proof}
We follow the same idea of the proof of Theorem 10 in \cite{ward2018adagrad}. First, we have
\begin{equation}\label{equ2-020}
-\langle \nabla f(x_t), \eta_t g_t \rangle = -\langle \nabla f(x_t), \hat{\eta}_t g_t \rangle 
+ \langle\nabla f(x_t), (\hat{\eta}_t - \eta_t)g_t \rangle.
\end{equation}
Note that $\hat{\eta}_t$ is independent of $g_t$, and $\mathbb{E}_t g_t = \nabla f(x_t)$ by assumption (\textbf{A1}). Hence, 
\begin{equation}
\mathbb{E}_t \langle \nabla f(x_t), \hat{\eta}_t g_t \rangle 
= \langle \nabla f(x_t), \hat{\eta}_t \nabla f(x_t) \rangle 
= \norm{\nabla f(x_t)}_{\hat{\eta}_t}^2.
\end{equation}
Taking the conditional expectation of Eq.~\eqref{equ2-020} with respect to $\xi_t$ while $\xi_1, \ldots, \xi_{t-1}$ being fixed, we have
\begin{equation}\label{equ2-022}
-\mathbb{E}_t \langle \nabla f(x_t), \eta_t g_t \rangle = - \norm{\nabla f(x_t)}^2_{\hat{\eta}_t} 
+ \mathbb{E}_t \langle \nabla f(x_t), (\hat{\eta}_t - \eta_t)g_t \rangle.
\end{equation}
To estimate the second term of Eq.~\eqref{equ2-022}, we first have
\begin{equation}\label{equ2-023}
\langle \nabla f(x_t), (\hat{\eta}_t-\eta_t)g_t \rangle
\leq \sum_{k=1}^d |\hat{\eta}_{t,k}-\eta_{t,k}||\nabla_k f(x_t)||g_{t,k}|.
\end{equation}
Let $V_t = \frac{1}{\bar{a}_t}\sum_{i=1}^t a_i g_i^2$, and $\hat{V}_t =  \frac{1}{\bar{a}_t}\left(\sum_{i=1}^{t-1} a_i g_i^2 + a_t \sigma_t^2\right)$. Then $\eta_t = \eta/(\sqrt{V_t} + \bm{\epsilon})$ and $\hat{\eta}_t = \eta/(\sqrt{\hat{V}_t} + \bm{\epsilon})$. It follows that 
\begin{equation}\label{equ2-024}
\begin{split}
|\hat{\eta}_t - \eta_t|
& = \frac{\eta |V_t -\hat{V}_t|}{(\sqrt{V_t}+\bm{\epsilon})(\sqrt{\hat{V}_t}+\bm{\epsilon})(\sqrt{V_t} + \sqrt{\hat{V}_t})} \\
& \leq \frac{\eta_t \hat{\eta_t}}{\eta} \frac{(a_t/\bar{a}_t)|g_t^2 - \sigma_t^2|}{\sqrt{V_t} + \sqrt{\hat{V}_t}}
\leq \frac{\sqrt{(a_t/\bar{a}_t)}}{\eta} \hat{\eta}_t \eta_t (|g_t|+\sigma_t).
\end{split}
\end{equation}
Note that the above inequality is coordinate-wise. 
By Eq.~\eqref{equ2-023} and Eq.~\eqref{equ2-024}, we have
\begin{equation}\label{equ8}
\begin{split}
&\langle \nabla f(x_t), (\hat{\eta}_t-\eta_t)g_t \rangle\\
\leq & \sum_{k=1}^d \frac{\sqrt{(a_t/\bar{a}_t)}}{\eta}\hat{\eta}_{t,k} \eta_{t,k}|\nabla_k f(x_t)| |g_{t,k}|^2 \\
&+ \sum_{k=1}^d \frac{\sqrt{(a_t/\bar{a}_t)}}{\eta}\sigma_{t,k}\hat{\eta}_{t,k}\eta_{t,k}|\nabla_k f(x_t)||g_{t,k}|.
\end{split}
\end{equation}
We claim that 
\begin{equation}\label{equ5-027}
\begin{split}
&\mathbb{E}_t \sum_{k=1}^d \frac{\sqrt{(a_t/\bar{a}_t)}}{\eta}\hat{\eta}_{t,k}\eta_{t,k}|\nabla_k f(x_t)||g_{t,k}|^2 \\
\leq\ & \frac{1}{4}\norm{\nabla f(x_t)}^2_{\hat{\eta}_t} + \frac{\sigma}{\eta}\mathbb{E}_t \left[\sqrt{(a_t/\bar{a}_t)}\norm{\eta_t g_t}^2\right].
\end{split}
\end{equation}
To see this, first, if $\sigma_{t,k} > 0$,  we apply the arithmetic inequality $2ab \leq a^2 + b^2$ with 
\begin{align}
a = \frac{1}{2\sigma_{t,k}}\sqrt{\hat{\eta}_{t,k}}|\nabla_k f(x_t)||g_{t,k}| \\
b = \frac{\sqrt{(a_t/\bar{a}_t)}}{\eta}\sigma_{t,k}\sqrt{\hat{\eta}_{t,k}}\eta_{t,k}|g_{t,k}|    
\end{align}
to the left-hand side of Eq.~\eqref{equ5-027}, arriving at
\begin{equation}\label{equ13}
\begin{split}
& \frac{\sqrt{(a_t/\bar{a}_t)}}{\eta}\hat{\eta}_{t,k}\eta_{t,k}|\nabla_k f(x_t)||g_{t,k}|^2 \\
\leq\ & \frac{|g_{t,k}|^2}{4\sigma_{t,k}^2}\hat{\eta}_{t,k}|\nabla_k f(x_t)|^2 
+ \frac{(a_t/\bar{a}_t)}{\eta^2}\sigma_{t,k}^2\hat{\eta}_{t,k}|\eta_{t,k}g_{t,k}|^2.
\end{split}
\end{equation}
Note that $\mathbb{E}_t |g_{t,k}|^2 = \sigma_{t,k}$, 
\[\hat{\eta}_{t,k} = \frac{\eta}{\sqrt{\frac{1}{\bar{a}_t}(\sum_{i=1}^{t-1} a_i g_i^2 + a_t \sigma_{t,k}^2)} + \bm{\epsilon}} \leq \frac{\eta}{\sqrt{(a_t/\bar{a}_t)}\sigma_{t,k}},\]
and $\sigma_{t,k} \leq \sqrt{\mathbb{E}_t\norm{g_t}^2} \leq \sigma$ by assumption (\textbf{A2}). Therefore,
\begin{equation}\label{equ2-027}
\begin{split}
&\mathbb{E}_t \left[ \frac{\sqrt{(a_t/\bar{a}_t)}}{\eta}\hat{\eta}_{t,k}\eta_{t,k}|\nabla_k f(x_t)||g_{t,k}|^2 \right]\\
\leq\ & \frac{1}{4} \hat{\eta}_{t,k}|\nabla_k f(x_t)|^2 
+ \frac{1}{\eta} \mathbb{E}_t \left[\sigma_{t,k}\sqrt{(a_t/\bar{a}_t)}|\eta_{t,k} g_{t,k}|^2\right] \\
\leq\ & \frac{1}{4} \hat{\eta}_{t,k} |\nabla_k f(x_t)|^2 + \frac{\sigma}{\eta}\mathbb{E}_t \left[\sqrt{(a_t/\bar{a}_t)}|\eta_{t,k} g_{t,k}|^2\right].
\end{split}
\end{equation}
On the other hand, if $\sigma_{t,k} = 0$, then $g_{t,k} = 0$, and Eq.~\eqref{equ2-027} holds automatically. By taking sum of the components 
for $k= 1, 2, \ldots, d$, we then obtain the desired claim in Eq.~\eqref{equ5-027}.

Similarly, we apply the arithmetic inequality with 
\begin{align}
a = \frac{1}{2}\sqrt{\hat{\eta}_{t,k}}|\nabla_k f(x_t)|\\
b = \frac{\sqrt{(a_t/\bar{a}_t)}}{\eta} \sigma_{t,k}\sqrt{\hat{\eta}_{t,k}}\eta_{t,k}|g_{t,k}|
\end{align}
to the second term of Eq.~\eqref{equ8}, arriving at
\begin{equation}
\begin{split}
&\frac{\sigma_{t,k}}{\eta}\hat{\eta}_{t,k}\eta_{t,k}|\nabla_k f(x_t)||g_{t,k}| \\ 
\leq\ & \frac{1}{4}\hat{\eta}_{t,k}|\nabla_k f(x_t)|^2 
+ \frac{(a_t/\bar{a}_t)}{\eta^2}\sigma_{t,k}^2\hat{\eta}_{t,k}|\eta_{t,k}g_{t,k}|^2 \\
\leq\ & \frac{1}{4}\hat{\eta}_{t,k}|\nabla_k f(x_t)|^2 + \frac{\sigma}{\eta}\sqrt{(a_t/\bar{a}_t)}|\eta_{t,k}g_{t,k}|^2.
\end{split}
\end{equation}
Hence,
\begin{align}\label{equ2-029}
&\mathbb{E}_t \sum_{k=1}^d \frac{\sigma}{\eta}\hat{\eta}_{t,k}|\nabla_k f(x_t)||g_{t,k}| \nonumber\\
&\leq \frac{1}{4}\norm{\nabla f(x_t)}_{\hat{\eta}_t} + \frac{\sigma}{\eta} \mathbb{E}_t \left[\sqrt{(a_t/\bar{a}_t)}\norm{\eta_t g_t}^2\right].
\end{align}
Combining Eq.~\eqref{equ8}, Eq.~\eqref{equ2-027}, and Eq.~\eqref{equ2-029}, we obtain the following estimate 
\begin{align}\label{equ2-030}
&\mathbb{E}_t \langle \nabla f(x_t), (\hat{\eta}_t-\eta_t)g_t\rangle \nonumber\\
&\leq \frac{1}{2}\norm{\nabla f(x_t)}^2_{\hat{\eta}_t} + \frac{2\sigma}{\eta} \mathbb{E}_t \left[\sqrt{(a_t/\bar{a}_t)}\norm{\eta_t g_t}^2\right].
\end{align}
The proof is finished by taking the estimate Eq.~\eqref{equ2-030} into Eq.~\eqref{equ2-022}.
\end{proof}

\begin{lemma}\label{lem5}
Assume that the noisy gradient $g_t$ in each iteration satisfies assumptions (\textbf{A1}) and (\textbf{A2}). We have the following estimate
\begin{equation}
\mathbb{E}\sum_{t=1}^T(a_t/\bar{a}_t)\norm{\eta_t g_t}^2 \leq \eta^2 d\ \log\left(1 + \frac{\sigma^2}{\epsilon}\sum_{t=1}^T a_t\right).
\end{equation}
\end{lemma}

\begin{proof}
Let $V_t = \frac{1}{\bar{a}_i}\sum_{i=1}^t a_t g_i^2$. Then $\eta_t = \eta / (\sqrt{V_t} + \bm{\epsilon})$. We have that
\begin{equation}
\begin{aligned}
\sum_{t=1}^T (a_t/\bar{a}_t)\norm{\eta_t g_t}^2 
&= \eta^2 \sum_{t=1}^T\sum_{k=1}^d \frac{(a_t/\bar{a}_t)g_{t,k}^2}{(\sqrt{V_{t,k}} + \epsilon)^2}\\
&\leq \eta^2 \sum_{t=1}^T \sum_{k=1}^d \frac{(a_t/\bar{a}_t)g_{t,k}^2}{\epsilon^2 + V_{t,k}}\\
&=\eta^2 \sum_{k=1}^d\sum_{t=1}^T \frac{a_t g_{t,k}^2}{\epsilon^2 \bar{a}_t +\sum_{i=1}^t a_i g_{i,k}^2}\\
&\leq \eta^2 \sum_{k=1}^d\sum_{t=1}^T \frac{a_t g_{t,k}^2}{\epsilon^2 + \sum_{i=1}^t a_i g_{i,k}^2}.
\end{aligned}
\end{equation}
The last inequality is due to $\bar{a}_t \geq a_1 = 1$. By Lemma \ref{lem1}, we have
$$
\sum_{t=1}^T \frac{a_t g_{t,k}^2}{\epsilon^2 + \sum_{i=1}^t a_i g_{i,k}^2} \leq \log\left(\epsilon^2 + \sum_{i=1}^T a_i g_{i,k}^2\right) - \log(\epsilon^2).
$$
On the other hand, since $\log(x)$ is concave, we have
\begin{align}
\mathbb{E}\left[\log\left(\epsilon^2 + \sum_{i=1}^T a_i g_i^2 \right)\right] 
&\leq \log\left(\mathbb{E}\left[\epsilon^2 + \sum_{i=1}^T a_i g_i^2\right]\right) \\
& \leq \log\left(\epsilon^2 + \sigma^2 \sum_{i=1}^T a_i\right).
\end{align}
Hence,
\begin{equation}
\begin{split}
&\mathbb{E}\sum_{t=1}^T \left[(a_t/\bar{a}_t)\norm{\eta_t g_t}^2 \right] \\
\leq & \eta^2 \sum_{k=1}^d \left(\log\left(\epsilon^2 + \sigma^2 \sum_{t=1}^T a_t\right) -\log(\epsilon^2)\right) \\
\leq & \eta^2 d\ \log\left(1 + \frac{\sigma^2}{\epsilon^2}\sum_{t=1}^T a_t\right).
\end{split}
\end{equation}
\end{proof}

\begin{lemma}\label{lem6}
Let $\hat{\eta}_t$ be defined as in Lemma \ref{lem3}. Assume that the noisy gradients $g_t$ satisfy assumptions (\textbf{A1}) and (\textbf{A2}). Let $\tau$ be randomly selected from $\{1, 2,\ldots T\}$ with equal probability $\mathcal{P}(\tau = t) = 1/T$. We have the following estimate
\begin{equation}
\left(\mathbb{E}\norm{f(x_\tau)}^{4/3}\right)^{3/2} 
\leq \frac{\sqrt{2\epsilon^2 +2\sigma^2T}}{\eta T}\ \mathbb{E}\sum_{t=1}^T\norm{\nabla f(x_t)}^2_{\hat{\eta}_t}.
\end{equation}
\end{lemma}

\begin{proof}
We follow the proof from \cite{ward2018adagrad} with modification for the coordinate-wise case. Let $\hat{V}_t = \frac{1}{\bar{a}_t}\left(\sum_{i=1}^{t-1} g_i^2 + a_t\sigma_t^2\right)$, where $\sigma_t = \sqrt{\mathbb{E}_t g_t^2}$, we have $\hat{\eta}_{t,k} = \eta/(\sqrt{\hat{V}_{t,k}} + {\epsilon})$. 
By H\"{o}lder's inequality we have $\mathbb{E}|XY| \leq (\mathbb{E}|X|^p)^{1/p}(\mathbb{E}|Y|^q)^{1/q}$ for any $0 < p,\ q < 1$ with $1/p+1/q =1$. 
Now taking $p= 3/2$ and $q=3$, and
$$X = \left(\frac{\norm{\nabla f(x_t)}^2}{\epsilon + \sqrt{\norm{\hat{V}_t}}_1}\right)^{2/3},\quad Y = \left(\epsilon + \sqrt{\norm{\hat{V}_t}_1}\right)^{2/3},$$
we have
\begin{equation*}
\mathbb{E} \norm{\nabla f(x_t)}^{4/3} \leq \left(\mathbb{E} \frac{\norm{\nabla f(x_t)}^2}{\epsilon + \sqrt{\norm{\hat{V}_t}_1}}\right)^{\frac{2}{3}}
\left(\mathbb{E}\left[\left(\epsilon + \sqrt{\norm{\hat{V}_t}_1}\right)^2\right]\right)^{\frac{1}{3}}.
\end{equation*}
Namely,
\begin{equation}\label{equ3-050}
\left(\mathbb{E}\norm{\nabla f(x_t)}^{4/3}\right)^{3/2} \leq \left(\mathbb{E}\frac{\norm{\nabla f(x_t)}^2}{\epsilon + \sqrt{\norm{\hat{V}_t}_1}}\right)\left(\mathbb{E}\left[\left(\epsilon+\norm{\hat{V}_t}_1\right)^2\right]\right)^{1/2}.
\end{equation}
Note that 
\begin{equation}\label{equ3-051}
\begin{split}
\frac{\norm{\nabla f(x_t)}^2}{\epsilon + \sqrt{\norm{\hat{V}_t}_1}} 
 = & \sum_{k=1}^d \frac{|\nabla_k f(x_t)|^2}{\epsilon + \sqrt{\norm{\hat{V}_t}_1}}
 \leq \frac{1}{\eta} \sum_{k=1}^d \frac{\eta|\nabla_k f(x_t)|^2}{\epsilon + \sqrt{\hat{V}_{t,k}}} \\
  &= \frac{1}{\eta} \sum_{k=1}^d \hat{\eta}_{t,k} |\nabla_k f(x_t)|^2
= \frac{1}{\eta} \norm{\nabla f(x_t)}_{\hat{\eta}_t}^2.
\end{split}
\end{equation}
On the other hand, for any $t \leq T$ we have
\begin{equation}\label{equ3-052}
\begin{split}
& \mathbb{E}\left[\left(\epsilon + \sqrt{\norm{\hat{V}_t}_1} \right)^2\right] \\
&\leq \mathbb{E}\left[2\left(\epsilon^2 + \norm{\hat{V}_t}_1\right)\right]\\
&= 2\epsilon^2 + 2\mathbb{E} \sum_{k=1}^d \frac{1}{\bar{a}_t}\left(\sum_{i=1}^{t-1} a_i g_{i,k}^2 + a_t \sigma_{t,k}^2\right) \\
&= 2\epsilon^2 + 2\sum_{i=1}^t (a_i/\bar{a}_t)\mathbb{E}\norm{g_i}^2 \\
&\leq 2\epsilon^2 + 2\sigma^2 t \leq 2\epsilon^2 + 2\sigma^2 T.
\end{split}
\end{equation}
Hence, by Eq.~\eqref{equ3-050}, Eq.~\eqref{equ3-051}, and Eq.~\eqref{equ3-052}, we have 
\begin{equation}
\left(\mathbb{E}\norm{\nabla f(x_t)}^{4/3} \right)^{3/2} 
\leq \frac{\sqrt{2\epsilon^2 + 2\sigma^2 T}}{\eta}\ \mathbb{E}\norm{\nabla f(x_t)}_{\hat{\eta}_t}^2,\ \forall t \leq T.
\end{equation}
It follows that
\begin{equation}
\begin{split}
\left(\mathbb{E}\norm{f(x_\tau)}^{4/3}\right)^{3/2}
&= \left(\frac{1}{T}\sum_{t=1}^T\mathbb{E} \norm{\nabla f(x_t)}^{4/3}\right)^{3/2} \\
&\leq \frac{1}{T}\sum_{t=1}^T\left(\mathbb{E} \norm{\nabla f(x_t)}^{4/3}\right)^{3/2} \\
&\leq \frac{\sqrt{ 2\epsilon^2 + 2\sigma^2 T}}{\eta T} \mathbb{E}\sum_{t=1}^T \norm{\nabla f(x_t)}_{\hat{\eta}_t}^2.
\end{split}
\end{equation}
The proof is completed.
\end{proof}

\subsection{Proof of Theorem \ref{AdaGrad-high-probility}}\label{main proof}
In this section, we prove our main theorem \ref{AdaGrad-high-probility}. We restate the theorem for readers' convenience.
\begin{theorem*}
Let $\{x_t\} \subseteq \mathbb{R}^{d}$ be a sequence generated by AdaUSM. 
Assume that the noisy gradient $g_{t}$ in each iteration satisfies assumptions (\textbf{A1}) and (\textbf{A2}).
Suppose that the sequence of weights $\{a_t\}$  is non-decreasing.
Let $\tau$ be randomly selected from $\{1, 2,\ldots T\}$ with equal probability $\mathcal{P}(\tau = t) = 1/T$. Then we have the following estimate
\begin{equation*}
\big(\mathbb{E}\left[\norm{\nabla f(x_t)}^{4/3}\right]\big)^{3/2} \leq Bound(T) = \mathcal{O}\left(\redt{d}{\log T}/{\sqrt{T}}\right),
\end{equation*}
where 
$$Bound(T) = \frac{\sqrt{2\epsilon^2 + 2\sigma^2 T}}{T} \big(C_1 + C_2\log\big(1 + \frac{\sigma^2}{\epsilon^2}\sum_{t=1}^T a_t\big)\big)$$
with $C_1 = \frac{2(f(x_1) -f^*)}{(1+\lambda\mu)\eta}$ and $C_2 = \left(\frac{2\eta(1+2\lambda)^2 L d}{(1+\lambda\mu)(1-\mu)^3}+\frac{4\sigma d}{1-\mu}\right)$.
\end{theorem*}

The key ingredient of the proof of the theorem is the following estimate which we will prove later.
\begin{lemma}\label{lem7}
Assume the same setting as Theorem \ref{AdaGrad-high-probility}. We have the following estimate
\begin{equation}\label{equ19}
\begin{split}
& \mathbb{E}\sum_{t=1}^T \norm{\nabla f(x_t)}^2_{\hat{\eta}_t}  \leq \frac{2(f(x_1) - f^*)}{1+\lambda\mu} 
\\
&+ 
\big(\frac{2(1+2\lambda)^2L}{(1+\lambda\mu)(1-\mu)^3} +\frac{4\sigma}{\eta(1-\mu)}\big)\mathbb{E}\sum_{t=1}^T (a_t/\bar{a}_t)\norm{\eta_t g_t}^2,
\end{split}
\end{equation}
where $\hat{\eta}_t$ is defined as in Lemma \ref{lem4}.
\end{lemma}
 
\begin{proof}[Proof of Lemma \ref{lem7}]
Since $x_{t+1} = x_t + m_t + \lambda\mu (m_t - m_{t-1})$, it follows by Lipschitz continuity of the gradient of $f$ and the descent lemma in \cite{nesterov2013introductory} that
\begin{align} \label{equ20}
f(x_{t+1}) \leq& f(x_t) + \langle \nabla f(x_t), m_t + \lambda \mu (m_t - m_{t-1})\rangle \nonumber \\
&+ \frac{L}{2}\norm{m_t + \lambda \mu (m_t - m_{t-1})}^2.
\end{align}
Since $m_t = \mu m_{t-1} - \eta_t g_t$, it follows that
\begin{equation}\label{equ4-021}
\begin{split}
&f(x_{t+1}) \\
& \leq f(x_t) + \langle \nabla f(x_t), (1+\lambda\mu -\lambda) m_t - \lambda \eta_t g_t \rangle\\
&\quad + \frac{L}{2} \norm{m_t + \lambda \mu (m_t - m_{t-1})}^2 \\
& \leq f(x_t) + (1+\lambda\mu -\lambda)\langle \nabla f(x_t), m_t\rangle - \lambda \langle \nabla f(x_t), \eta_t g_t \rangle \\
& \quad + \frac{L}{2} \norm{m_t + \lambda \mu (m_t - m_{t-1})}^2.
\end{split}
\end{equation}
By Lemma \ref{lem3}, we have 
\begin{align}\label{equ21}
& \langle \nabla f(x_t), m_t\rangle \nonumber \\
&\leq (1+2\lambda)L\sum_{i=1}^{t-1} \norm{m_{i}}^2\mu^{t-i} - \sum_{i=1}^t \langle \nabla f(x_i), \eta_i g_i\rangle \mu^{t-i}.
\end{align}
Note that $1+\lambda\mu -\lambda \geq 0$ since $\lambda \leq 1/(1-\mu)$. Combining Eq.~\eqref{equ4-021} and Eq.~\eqref{equ21}, we have
\begin{equation} \label{equ22}
\begin{split}
&f(x_{t+1}) \\
& \leq f(x_t) + (1+\lambda\mu -\lambda)(1+2\lambda) L \sum_{i=1}^{t-1}\norm{m_{i}}^2\mu^{t-i} \\
&\ \quad + \frac{L}{2} \norm{m_t + \lambda \mu (m_t - m_{t-1})}^2 \\
&\ \quad - (1+\lambda\mu-\lambda)\sum_{i=1}^t \langle \nabla f(x_i), \eta_i g_i \rangle \mu^{t-i} - \lambda \langle \nabla f(x_t), \eta_t g_t \rangle.
\end{split}
\end{equation}
On one hand, by arithmetic inequality, we have
\begin{equation}
\norm{m_t + \lambda \mu(m_t - m_{t-1})}^2 \leq 2(1+\lambda\mu)^2 \norm{m_t}^2 + 2(\lambda\mu)^2 \norm{m_{t-1}}^2.
\end{equation}
Hence,
\begin{equation}\label{equ4-045}
\begin{split}
&\ (1+\lambda\mu -\lambda)(1+2\lambda) L \sum_{i=1}^{t-1}\norm{m_{i}}^2\mu^{t-i} \\
& \quad + \frac{L}{2} \norm{m_t + \lambda \mu (m_t - m_{t-1})}^2 \\
\leq &\ (1+\lambda\mu)(1+2\lambda) L \sum_{i=1}^{t-1}\norm{m_{i}}^2\mu^{t-i} - \lambda(1+2\lambda)\mu L \norm{m_{t-1}}^2 \\
&\ \quad + (1+\lambda\mu)^2 L \norm{m_t}^2 + (\lambda\mu)^2 L \norm{m_{t-1}}^2 \\
\leq &\ (1+2\lambda)^2 L \sum_{i=1}^t \norm{m_i}^2 \mu^{t-i}.
\end{split}
\end{equation}
Summarizing Eq.~\eqref{equ22} and Eq.~\eqref{equ4-045}, we have the following cleaner inequality:
\begin{equation}\label{equ5-047}
\begin{split}
f(x_{t+1}) \leq  &\ f(x_t) + (1+2\lambda)^2 L \sum_{i=1}^t \norm{m_i}^2 \mu^{t-i} \\
&\ \quad - (1+\lambda\mu -\lambda) \sum_{i=1}^t \langle \nabla f(x_i), \eta_i g_i \rangle \mu^{t-i} \\
&\ \quad - \lambda \langle \nabla f(x_t), \eta_t g_t\rangle.
\end{split}
\end{equation}
On the other hand, by Lemma \ref{lem3}, we have that
\begin{equation} \label{equ23}
\begin{aligned}
&-\mathbb{E}_i\langle f(x_i), \eta_i g_i\rangle \\
&\leq -\frac{1}{2}\norm{\nabla f(x_i)}^2_{\hat{\eta}_i} + \frac{2\sigma}{\eta} \mathbb{E}_i \left[\sqrt{(a_i/\bar{a}_i)}\norm{\eta_i g_i}^2\right]\\
&\leq -\frac{1}{2}\norm{\nabla f(x_i)}^2_{\hat{\eta}_i} + \frac{2\sigma}{\eta} \mathbb{E}_i \left[{(a_i/\bar{a}_i)}\norm{\eta_i g_i}^2\right],\ \forall i.
\end{aligned}
\end{equation}
The second inequality is due to that $(a_t/\bar{a}_t) \geq 1$, so $\sqrt{(a_t/\bar{a}_t)} \leq (a_t/\bar{a}_t)$. Combining Eq.~\eqref{equ5-047} and Eq.~\eqref{equ23}, taking sum from $1$ to $T$ and taking expectation, followed by moving the gradient square terms to the left-hand side, we obtain that
\begin{equation}\label{equ24}
\begin{split}
&\ (1+\lambda\mu -\lambda)\mathbb{E}\sum_{t=1}^T \sum_{i=1}^t \frac{1}{2}\norm{\nabla f(x_i)}^2_{\hat{\eta}_i} \mu^{t-i}\\
&\qquad + \lambda \mathbb{E}\sum_{t=1}^T \frac{1}{2}\norm{\nabla f(x_t)}_{\hat{\eta}_t}^2 \\
\leq &\ f(x_1) - f^* + (1+2\lambda)^2 L\mathbb{E}\sum_{t=1}^T\sum_{i=1}^{t}\norm{m_i}^2\mu^{t-i} \\
&\ + \frac{2\sigma}{\eta}\left[ (1+\lambda\mu-\lambda)\mathbb{E}\sum_{t=1}^T\sum_{i=1}^t (a_i/\bar{a}_i)\norm{\eta_i g_i}^2\mu^{t-i} \right] \\
&\ + \frac{2\sigma}{\eta}\left[\lambda \mathbb{E} \sum_{t=1}^T (a_t/\bar{a}_t)\norm{\eta_t g_t}^2 \right]\\
\leq &\ f(x_1) - f^* + (1+2\lambda)^2 L\ \mathbb{E}\sum_{t=1}^T\sum_{i=1}^t \norm{m_i}^2\mu^{t-i} \\
&\ \quad + \frac{2\sigma(1+\lambda\mu)}{\eta} \mathbb{E} \sum_{t=1}^T\sum_{i=1}^t (a_i/\bar{a}_i)\norm{\eta_i g_i}^2\mu^{t-i}.
\end{split}
\end{equation}
The last inequality is due to 
$ \lambda\sum_{t=1}^T (a_t/\bar{a}_t)\norm{\eta_t g_t}^2 \leq \lambda\sum_{t=1}^T \sum_{i=1}^t (a_i/\bar{a}_i)\norm{\eta_i g_i}^2\mu^{t-i}. $
Similarly, for the left-hand side, note that $1+\lambda\mu -\lambda \geq 0$ since $\lambda \leq 1/(1 - \mu)$, we have that
\begin{equation}\label{equ25}
\begin{split}
&\ (1+\lambda\mu-\lambda)\sum_{t=1}^T \sum_{i=1}^t \norm{\nabla f(x_i)}^2_{\hat{\eta}_i}\mu^{t-i} 
+ \lambda \sum_{t=1}^T \norm{\nabla f(x_t)}^2_{\hat{\eta}_t} \\
\geq &\ (1+\lambda\mu - \lambda)\sum_{t=1}^T \norm{\nabla f(x_t)}^2_{\hat{\eta}_t} 
+ \lambda \sum_{t=1}^T \norm{\nabla f(x_t)}^2_{\hat{\eta}_t} \\
= &\ (1+\lambda\mu) \sum_{t=1}^T \norm{\nabla f(x_t)}^2_{\hat{\eta}_t} .
\end{split}
\end{equation} 
Now we are left to estimate the third term and the last term in the right-hand side of Eq.~\eqref{equ24}. We apply the double-sum trick:
\begin{align}\label{equ26}
\sum_{t=1}^T\sum_{i=1}^{t} \norm{m_i}^2\mu^{t-i} 
&= \sum_{i=1}^{T}\sum_{t=i}^T\norm{m_i}^2\mu^{t-i} \nonumber\\
&\leq \frac{1}{1-\mu}\sum_{i=1}^{T}\norm{m_i}^2,
\end{align}
\begin{align}\label{equ27}
\sum_{t=1}^T\sum_{i=1}^{t}\norm{\eta_i g_i}^2 \mu^{t-i} &= \sum_{i=1}^T\sum_{t=i}^T \norm{\eta_i g_i}^2\mu^{t-i} \nonumber\\
&\leq \frac{1}{1-\mu}\sum_{i=1}^T \norm{\eta_i g_i}^2.
\end{align}
Combining Eq.~\eqref{equ24}, Eq.~\eqref{equ25}, Eq.~\eqref{equ26}, and Eq.~\eqref{equ27}, we have
\begin{align}\label{equ28}
& \frac{1+\lambda\mu}{2}\mathbb{E}\left[\sum_{t=1}^T \norm{\nabla f(x_t)}^2_{\hat{\eta}_t}\right] \\
& \leq f(x_1)-f^* + \frac{(1+2\lambda)^2L}{1-\mu} \mathbb{E}\left[\sum_{t=1}^T\norm{m_t}^2\right] \\
&\ \quad + \frac{2\sigma(1+\lambda\mu)}{\eta(1-\mu)} \mathbb{E}\left[\sum_{t=1}^T (a_t/\bar{a}_t)\norm{\eta_t g_t}^2\right].
\end{align}
Finally, by Lemma \ref{lem2} and $(a_t/\bar{a}_t) \geq 1$ we have
\begin{align}\label{equ29}
\mathbb{E}\left[\sum_{t=1}^T \norm{m_t}^2\right] 
&\leq \frac{1}{(1-\mu)^2}\mathbb{E}\sum_{t=1}^T\norm{\eta_t g_t}^2 \nonumber \\
&\leq \frac{1}{(1-\mu)^2} \mathbb{E}\left[\sum_{t=1}^T(a_t/\bar{a}_t)\norm{\eta_t g_t}^2\right].
\end{align}
Combining Eq.~\eqref{equ28} and Eq.~\eqref{equ29}, we obtain
\begin{equation}
\begin{split}
& \mathbb{E}\sum_{t=1}^T \norm{\nabla f(x_t)}^2_{\hat{\eta}_t}  \leq \frac{2(f(x_1)-f^*)}{1+\lambda\mu} \\
&+ \big(\frac{2(1+2\lambda)^2L}{(1+\lambda\mu)(1-\mu)^3} + \frac{4\sigma}{\eta(1-\mu)}\big)\mathbb{E}\sum_{t=1}^T (a_t/\bar{a}_t)\norm{\eta_t g_t}^2.
\end{split}
\end{equation}
The proof is completed.
\end{proof}

We ultimately prove Theorem \ref{AdaUSM-high-probility}.
\begin{proof}[Proof of Theorem]
The theorem is an immediate result of Lemma \ref{lem7}, Lemma \ref{lem5}, and Lemma \ref{lem6}.
\end{proof}

\end{document}